\pdfoutput=1
\documentclass[lettersize,journal]{IEEEtran}
\usepackage{amsmath,amssymb}
\usepackage{graphicx,graphics,color,psfrag}
\usepackage{cite,balance}
\usepackage{caption}
\allowdisplaybreaks
\usepackage{algorithm}
\usepackage{accents}
\usepackage{amsthm}
\usepackage{hyperref}
\usepackage{amsmath}
\usepackage{bm}
\usepackage{algorithmic}
\usepackage[english]{babel}
\usepackage{multirow}
\usepackage{enumerate}
\usepackage{cases}
\usepackage{stfloats}
\usepackage{dsfont}
\usepackage{color,soul}
\usepackage{amsfonts}
\usepackage{cite,graphicx,amsmath,amssymb}
\usepackage{framed}
\usepackage{fancyhdr}
\usepackage{hhline}
\usepackage{ulem}
\usepackage{graphicx,graphics}
\usepackage{array,color}
\usepackage{amsmath}
\usepackage{booktabs}
\usepackage{subfigure}
\usepackage[nolist]{acronym}
\usepackage{amsmath, amsfonts, amssymb, amsthm}

\makeatletter
\renewcommand{\fnum@figure}{Fig. \thefigure}
\makeatother

\newtheoremstyle{itshape}
{.0\baselineskip\@plus.0\baselineskip\@minus.0\baselineskip}
{.0\baselineskip\@plus.0\baselineskip\@minus.0\baselineskip}
{\itshape}
{}
{\bfseries}
{.}
{ }
{}

\theoremstyle{itshape}

\newtheorem{theorem}{Theorem}

\newtheorem{assumption}{Assumption}

\renewcommand{\algorithmicrequire}{ \textbf{Input:}}
\renewcommand{\algorithmicensure}{ \textbf{Output:}}
\begin{document}
	\title{Accelerating Split Federated Learning over Wireless Communication Networks}
	
	\author{ Ce Xu, Jinxuan Li, Yuan Liu, Yushi Ling, and Miaowen Wen
		
		\thanks{C. Xu, Y. Liu and M. Wen are with school of Electronic and Information Engineering, South China University of Technology, Guangzhou 510641, China (e-mails: eecexu@mail.scut.edu.cn, eeyliu@scut.edu.cn,  eemwwen@scut.edu.cn). J. Li and Y. Ling are with Guangzhou Power Supply Bureau, Guangdong Power Grid Co., Ltd., CSG, Guangzhou 510620, China (e-mails: billywoof@163.com,   546190218@qq.com). \emph{Corresponding author: Yuan Liu.}
			
		This paper was supported in part by the Science and Technology Project of China Southern Power Grid Corporation under Grant GDKJXM20220333.
		}
	}
	
	\maketitle
	
	\vspace{-1.5cm}
	\begin{abstract}

	The development of artificial intelligence (AI) provides opportunities for the promotion of deep neural network (DNN)-based applications. However, the large amount of parameters and computational complexity of DNN makes it difficult to deploy it on edge devices which are resource-constrained. An efficient method to address this challenge is model partition/splitting, in which DNN is divided into two parts which are deployed on device and server respectively for co-training or co-inference. In this paper, we consider a split federated learning (SFL) framework that combines the parallel model training mechanism of federated learning (FL) and the model splitting structure of split learning (SL). We consider a practical scenario of heterogeneous devices with individual split points of DNN. We formulate a joint problem of split point selection and bandwidth allocation to minimize the system latency. By using alternating optimization, we decompose the problem into two sub-problems and solve them optimally. Experiment results demonstrate the superiority of our work in latency reduction and accuracy improvement.
	\end{abstract}
	

	\section{Introduction}
	\subsection{Background}
	\IEEEPARstart{W}{ith} the rapid development of deep learning\cite{8736011}, many artificial intelligence (AI) applications have been widely used in practice, such as face recognition \cite{masi2018deep}, augment reality \cite{sahu2021artificial} and object detection \cite{szegedy2013deep}. In AI applications, machine learning (ML) models including convolutional neural network (CNN) or deep neural network (DNN) are trained by rich data generated by network edge (like sensors and internet-of-thing (IoT) devices), and the principle is to update the parameters of the model to minimize the error of the output results and finally establish a mapping function to make predictions on unknown data \cite{liu2017survey}. However, data is usually privately sensitive (such as medical and financial data), and as a result, enterprises or individuals may resist sharing their data to service providers to centralized training.
	
	As devices always have a certain level of computing capability, distributed learning enabling local training at device is more attractive since raw data of device is used for local training but kept at device locally to reduce privacy leakage risks. Federated learning (FL) \cite{mcmahan2017communication,li2020federated} is a promising distributed learning framework, in which a ML model is obtained with a central server aggregating local models trained by distributed devices. Specifically, each device first downloads a global model from server. Then each device trains a local model based on its local data and then transmits local model to server. Finally server aggregates all received local models to update global model. The above device-server iteration continues until convergence.

	However, FL brings up new challenges over wireless/edge networks, due to the limited computing and communication power on devices. On the one hand, ML model especially DNN often requires huge amount of parameters and calculation of the model. For example, AlexNet\cite{NIPS2012_c399862d} needs more than 200 MBytes of memory to store parameters, while that of VGG16 \cite{simonyan2014very} is as high as 500 MBytes. The intensive computation of DNN on devices such as IoT sensors and smart phones is not so effective. On the other hand, as DNN parameters are usually high-dimensional and thus the DNN transmission from devices to server suffers from high energy consumption and delay. 
	
	Split learning (SL) \cite{vepakomma2018split} is another distributed learning method to deal with the challenges raised by FL. According to the fact that DNN is composed of multiple layers, DNN can be partitioned into two parts vertically where the front-end part and the back-end part are executed on the device and server respectively. In training phase, device processes its raw data through the front-end part of DNN, and then sends the intermediate result to server to complete the forward propagation stage. Then, in backward propagation stage, server updates the parameters of the back-end part and sends back the intermediate gradients to device to finish the parameters update of the front-end part. In the whole process, the transmission between device and server is only the intermediate result of the split point, instead of the whole DNN parameters. By doing so, most of the computation is offloaded to server whose computing power is much stronger than that of device, so that the computing cost will be greatly saved. In addition, device does not need to transmit huge DNN model parameters but small intermediate result of a certain layer of DNN, the communication load will also be reduced.
	
	By integrating the collaborative training framework of FL and neural network splitting structure of SL, split federated learning (SFL) is showed to be more communication efficient \cite{thapa2022splitfed,turina2021federated,9923620}. However, these work only evaluate the performance of SFL experimentally, how to split DNN and improve communication efficiency for SFL is largely unknown. 

	\subsection{Related Works}
	
	Extensive works focus on delay and energy minimization, i.e., communication efficiency for FL. The transmission of DNN parameters from devices to server requires a great deal of communication resources, and DNN model compression can reduce communication load \cite{sattler2019robust,jiang2022model}. The authors in \cite{huang2022accelerating} propose a topology-optimized scheme to improve the both communication and computation efficiency. In addition, by device scheduling, the communication efficiency and training performance can be efficiently improved \cite{zhang2022communication,ren2020scheduling,shi2020joint}. As devices are energy constrained, resource allocation for energy efficiency is studied for FL \cite{yang2020energy,mo2021energy,do2021deep}. In order to overcome the communication bottleneck of FL in a multi-access channel, the authors in \cite{overtheair1,overtheair2} use over-the-air computing approach for model aggregation. Data heterogeneity is other critical issue in FL and attracts extensive attention \cite{yang2022client}.
	
	There is a handful of work using model splitting for inference. A usual solution is to establish a regression model that predicts the latency and energy consumption for each layer according to actual test, and then determine the optimal split point for inference tasks with the prediction result \cite{kang2017neurosurgeon,li2019edge}. Under different channel conditions, optimal DNN splitting is determined for both delay minimization and throughput maximization in \cite{8737614}. The work \cite{9758628} considers online DNN splitting for  edge-device co-inference. The authors in \cite{lyu2020foreseen} consider privacy preserving in SFL and inference. The authors in \cite{add3}  study the optimal multi-split points on multiple computing nodes.
	
	In addition, a deal of works focus on SL. The authors in \cite{gupta2018distributed} study collaborative training among multiple devices via SL. The work \cite{ParallelSL} proposes a parallel SL algorithm to reduce the training latency. Using the clustering method and combining the parallel algorithm, the authors in \cite{add2} propose a communication efficient SL architecture, and optimize resource allocation. The work \cite{add1} considers the dynamic optimization problem of split points in the collaborative DNN training scenario.

	Based on above related work, it is found that most effort on communication efficiency is devoted to either FL or SL, and how to improve communication efficiency for SFL has few attention. This problem is non-trivial due to two-folds: First, compared to DNN splitting used for inference, in SFL the sub-models need to be aggregated at server, and thus the selection of split points significantly affects the convergence performance. Second, due to the heterogeneity of multiple devices, including local data, fading channels, and communication-computing capacity, the resources of devices should be jointly optimized together with the split points.
	\subsection{Contributions}	
	Based on the above motivation, we study SFL in this paper and the main contribution is summarized as follows.

	\begin{itemize}
		\item We consider a new problem of jointly optimizing split point selection and bandwidth allocation in SFL for system total latency minimization. The devices have individual split points that significantly affect both the sub-model aggregation at server and the allocated bandwidth.
		\item The formulated problem is non-convex. To solve the problem, we decompose the problem into two sub-problems, in which the split point optimization problem is solved by backward induction and the bandwidth allocation problem is solved by convex method, and we use alternating optimization to iterate the two sub-problems to obtain the an efficient solution.
	\end{itemize}

	\subsection{Organization}
	The rest of the paper is organized as follows. Section \ref{sec:system_model_and_learning_mechanism} introduces the system model and the problem formulation. In Section \ref{sec:model_splitting_and_bandwidth_allocation}, we solve the split point optimization and bandwidth allocation problem. Experimental results are showed in Section \ref{sec:experiment}. Finally we have a conclusion in Section \ref{sec:Conclusion}.
	
	\section {System Model and Problem formulation}
	\label{sec:system_model_and_learning_mechanism}
	In this section, we first introduce the SFL model and communication model, respectively, then formulate the optimization problem of joint split point selection and bandwidth allocation. 
	
	


	\subsection{Split Federated Learning Model}
	\begin{figure}[ht]
		\subfigure[Split federated learning.]{
			\begin{minipage}[t]{0.5\linewidth}
				\centering
				\includegraphics[width=1.9\linewidth]{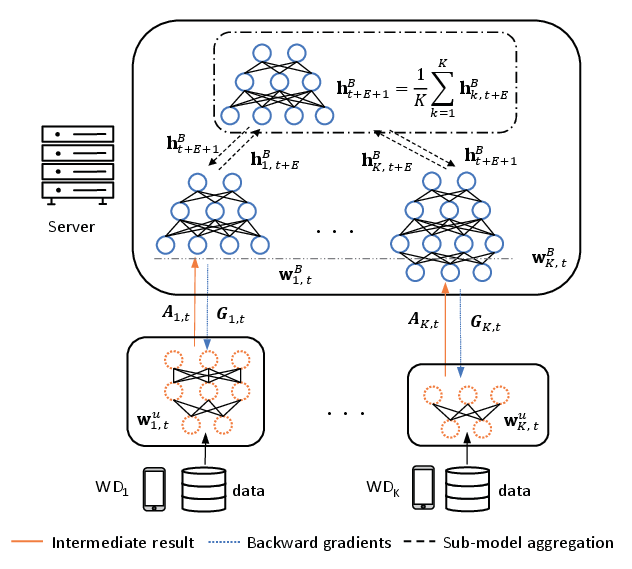}
				\vspace{-0.8cm}
				\quad\quad\quad
				\label{fig:system model}
			\end{minipage}
		\quad\quad\quad\quad\quad\quad\quad\quad\quad\quad\quad\quad\quad
		}%
	\\
		\subfigure[Illustration of model splitting with eight layers.]{
			\begin{minipage}[t]{0.5\linewidth}
				\centering
				\includegraphics[width=1.8\linewidth]{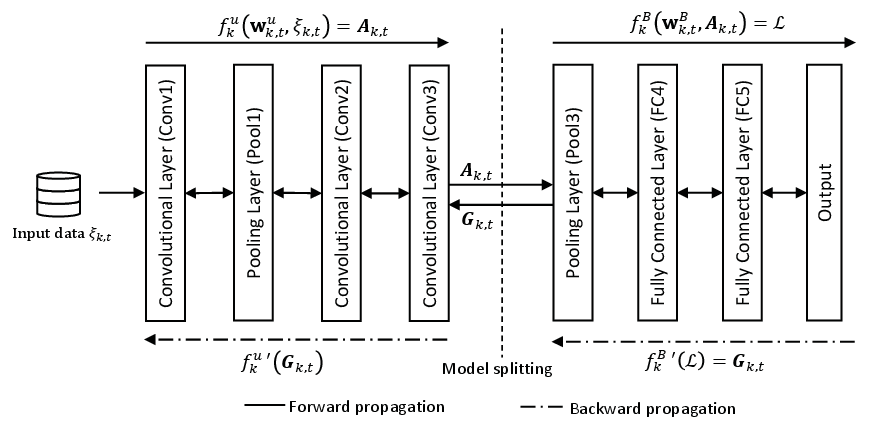}
				\vspace{-0.8cm}
				\quad\quad\quad\quad\quad
				\label{fig:model splitting}
			\end{minipage}%
		\quad\quad\quad\quad\quad\quad\quad\quad\quad\quad\quad
		}%
		\quad
		\centering
		\vspace{-0.4cm}
		\caption{Illustrated split federated learning and model splitting.}
		\label{fig:system}
		\vspace{-0cm}
	\end{figure}
	\label{sec:Federated learning with Model Splitting}
	As shown in Fig. \ref{fig:system model}, we consider a SFL system consisting of multiple devices and a server. The server connects with devices via wireless channels. Due to the limited computing capacity, the devices need to offload part of the training tasks to the server. In SFL, the devices and server collaboratively train a common DNN model but they only train a portion of the DNN model, as shown in Fig.  \ref{fig:system model}. Specifically, on each device the DNN is split to two parts, where the front-end part is deployed on the device and the other is deployed on the server. This implies that each device has an individual split point of the DNN, due to the fact that the devices have heterogeneous computing and transmission capacity. To tackle the issue of different split points of devices, for easy implementation, we assume that the server only aggregates the parts that are common to all back-end sub-models. Fig. \ref{fig:model splitting} illustrates an example of model splitting with eight layers. As the devices and server only have access to their own sub-networks, and thus it provides better privacy.
	
	
	
	Specifically, as shown in Fig. \ref{fig:system}, a device-server round of SFL consists of three steps as follows. 
	\begin{itemize}
		\item\textbf{Step 1:} Forward propagation. In iteration $t$, each device $k$ processes the input data with its own sub-network $\textbf{w}^{u}_{k,t}$, and sends the intermediate result $\bm{A}_{k,t}$ and label to the server through wireless channel. The server processes the intermediate result with the other sub-network $\textbf{w}^{B}_{k,t}$ to get the result and then calculates the loss $\mathcal{L}$ with the label.
		\item\textbf{Step 2:} Sub-model update. The server updates the parameter of its sub-network $\textbf{w}^{B}_{k,t}$ through backward propagation and sends the backward gradient $\bm{G}_{k,t}$ on the split layer to each device $k$. Then each device-side sub-network is updated in the same way. 
		\item\textbf{Step 3:} Sub-model aggregation on server. After $E$ iterations update, based on the highest split point, the common part of all sub-networks $\textbf{h}^B_{k,t+E}$ are aggregated on the server, and the updated $\textbf{h}^B_{t+E+1}$ will be used for next round training.
	\end{itemize}
Note that the mini-batch based stochastic gradient descent (SGD) method is adopted in this paper, so that Step 1 and Step 2 are repeated $E$ times before Step 3.

	Suppose that each device holds $M$ training data $x_{k,j}$, and the goal of training is to minimize the loss function, which is defined as:
	\begin{align}
		F_k(\textbf{w}_{k}^u,\textbf{w}_{k}^B)\triangleq\frac{1}{M}\sum^M_{j=1} f_k^B(\textbf{w}_{k}^B,f_k^u(\textbf{w}_{k}^u,x_{k,j})).
	\end{align}
	where $f_k^{u}$ and $f_k^{B}$ are loss functions for the front-end and back-end sub-models, respectively.

	In Step 1, the intermediate result $\bm{A}_{k,t}$ and the loss $\mathcal{L}$ in iteration $t$ are computed by
	\begin{align}
		\begin{aligned}
			\bm{A}_{k,t} =& f_k^{u}(\textbf{w}_{k,t}^{u},\xi_{k,t}),\\
			\mathcal{L} =& f_k^{B}(\textbf{w}_{k,t}^{B},\bm{A}_{k,t}),
		\end{aligned}
	\end{align}
 	where $\xi_{k,t}$ is a sample chosen from local data $x_{k,j}$. 
	

	In Step 2, the model parameters are updated by
	\begin{align}
		\begin{aligned}
			\left\{\begin{array}{ll}
				\textbf{w}_{k,t+1}^u&=\textbf{w}_{k,t}^u-\eta_t\frac{\partial \mathcal{L}}{\partial \textbf{w}_{k,t}^u}\\&=\textbf{w}_{k,t}^u-\eta_t\left[\nabla F_k(\textbf{w}_{k,t}^u,\textbf{w}_{k,t}^B,\xi_{k,t})\right]_u;\\
				\textbf{w}_{k,t+1}^B&=\textbf{w}_{k,t}^B-\eta_t\bm{G}_{k,t}\frac{\partial \bm{A}_{k,t}}{\partial \textbf{w}_{k,t+1}^B}\\&=\textbf{w}_{k,t}^B-\eta_t\left[\nabla F_k(\textbf{w}_{k,t}^u,\textbf{w}_{k,t}^B,\xi_{k,t})\right]_B,
			\end{array}
			\right.
		\end{aligned}
	\end{align}
	where $[\nabla F_k(\textbf{w}_{k,t}^u,\textbf{w}_{k,t}^B,\xi_{k,t})]_u$ and $[\nabla F_k(\textbf{w}_{k,t}^u,\textbf{w}_{k,t}^B,\xi_{k,t})]_B$ represent the gradients of model parameters deployed on device and server, respectively. 
	
	After completing a round, the common part of DNN model parameters deployed on server are aggregated in Step 3:
	\begin{align}
		\textbf{h}_{t+E+1}^B=\frac{1}{K}\sum_{k=1}^{K}\textbf{h}_{k,t+E}^B,
	\end{align}

From the processes described above, it is observed that the exchange between device-server in SFL is $\bm{A}_{k,t}$ and $\bm{G}_{k,t}$, whose sizes are significantly lower than in FL. For example, the sizes of $\bm{A}_{k,t}$ and $\bm{G}_{k,t}$ in AlexNet are about 280 KBytes, while the size of the exchanged model between device-server is over 200 MBytes.  Note that FL only needs to transmit once in each round while SFL needs multiple transmissions. However, the computation time saved by SFL compared to FL is significant, because most of the computation is done by the server, whereas in FL all the computation is undertook by the resource-limited devices. Therefore, even if SFL requires multiple local updates in a round, SFL can be more efficient than FL.
	
	\subsection{Convergence Analysis}
	
	
	Before analysis, we first clarify several symbols used in the analysis process to simplify the reasoning process and results. Because only the common part of back-end sub-models $\textbf{w}_{k,t}^B$ are aggregated, we use $\textbf{h}_{k,t}^u$ and $\textbf{h}_{k,t}^B$ to denote the parts of the model involved in aggregation and those not involved respectively, and the split point between them is $\ell = \max\ \{\ell_k\}.$
	
	We define the total model parameter for device $k$ as $\textbf{w}_{k,t}=\left[\textbf{h}_{k,t}^u,\textbf{h}_{k,t}^B\right]$. Assume the training task finishes after $T$ iterations and return ${\textbf{w}}_{k,T}$ as the output for device $k$, and $\overline{\textbf{w}}_{k,T}=[\textbf{h}^u_{k,T},\sum_{k=1}^K\frac{1}{K}\textbf{h}^B_{k,T}]$. Let $F^*$and $F_k^*$ be the minimum values of $F$ and $F_k$, respectively, and ${\bf w}^*$ the optimal parameters. 
	
	In order to simplify the analysis process, we make the following assumptions for loss function.
	\begin{assumption}
		\label{assumption1}
		We assume that the loss function of all devices obeys the following law:
		\begin{itemize}
			\item $F_k({\bf w})$ is Lipschitz smooth: $F_k({\bf v})\leq F_k({\bf w})+({\bf v}-{\bf w})^T\nabla{F_k({\bf w})+\frac{\beta}{2}\Vert {\bf v}-{\bf w}\Vert_2^2}$.
			\item $F_k({\bf w})$ is $\mu$-strongly convex: $F_k({\bf v})\geq F_k({\bf w})+({\bf v}-{\bf w})^T\nabla F_k({\bf w})+\frac{\mu}{2}\Vert {\bf v}-{\bf w}\Vert_2^2$.
			\item The variance of the gradients for each layer in device have upper bound: \\ $\mathbb{E}||\nabla F_k({\bf h}_{k,t}^u,{\bf h}_{k,t}^B,\xi_{k,t})-\nabla F_k({\bf h}_{k,t}^u,{\bf h}_{k,t}^B)||^2\leq L\sigma^2$, where $\xi_{k,t}$ is the data uniformly sampled from device $k$.
			\item The expected squared norm of each layer's gradients have upper bound: \\$\mathbb{E}||\nabla F_k({\bf h}_{k,t}^u,{\bf h}_{k,t}^B,\xi_{k,t})||^2\leq LZ^2$.
		\end{itemize}
	\end{assumption}
	
	Now, we derive the upper bound of the difference between the value of the loss function after $T$ iterations and the optimal value, and we have the following theorem:
	\begin{theorem}
		\label{theorem2}
		Let $\alpha = \frac{\beta}{\mu}, \gamma = \max\{8\alpha,E\}$ and choose the learning rate $\eta_t = \frac{2}{\mu(\gamma+t)}$. After $T$ iterations, there is an upper bound between the mean value and the optimal value of all device loss functions, which satisfies the following relationship:
		\begin{align}\label{eqn:convergence}
		\begin{aligned}
			&\mathbb{E}\left\{\sum_{k=1}^{K}\frac{1}{K}F(\overline{{\bf w}}_{k,T})\right\}-F^*\\
			&\leq \frac{\alpha}{\gamma+T}\left[\frac{2P}{\mu}+\frac{\mu}{2}(\gamma+1)\mathbb{E}\sum_{k=1}^K\frac{1}{K}\big{\Vert}\overline{{\bf w}}_{k,1}-{\bf w}^*\big{\Vert}^2\right],
		\end{aligned}
		\end{align}
		where $P=2(E-1)^2LZ^2+6\beta\Gamma+\ell Z^2+\frac{1}{K}(L-\ell)Z^2+\ell\sigma^2+\frac{1}{K}(L-\ell)\sigma^2$. And $\Gamma=\sum_{k=1}^K\frac{1}{K}\big(F^*-F_k^*\big)$ denotes the impact of data heterogeneity.
	\end{theorem}
	\begin{proof}
		See Appendix\ref{appendix2}.
	\end{proof}
	From \eqref{eqn:convergence}, we have 
	$\frac{\partial P}{\partial \ell} = (1-\frac{1}{K})(Z^2+\sigma^2)\geq0.$
	Therefore, higher split point $\ell_k$ leads to a lager upper bound of $\mathbb{E}\left\{\sum_{k=1}^{K}\frac{1}{K}F(\overline{{\bf w}}_{k,T})\right\}-F^*$, which means that learning performance can be compromised. This is due to that less DNN parameters participate in the aggregation process and thus are difficult to get more information for each device. In this way, the trained model may be lack of generalization ability.

	\subsection{Communication Model }
	Since the amount of computing and output size of each layer in DNN varies widely, the split point has a great impact on the execution efficiency of the SFL, including both computing and communication latency. In this paper, we aim at finding the best split points and allocating the optimal bandwidth for devices to achieve the minimum latency. In each iteration $t$, the total latency consists of two parts: computing latency and communication latency. The computing latency includes the latency generated by parameter updates on both devices and server. Moreover, the communication latency also includes latency generated by uplink transmission of $\bm{A}_{k,t}$ and downlink transmission of $\bm{G}_{k,t}$. Because the server has strong computing and communication capability, the latency generated by server can be ignored. As a result, we only consider the latency caused by devices since they are capability-limited in general. In following, we focus on latency in one iteration, thus the indices of iteration $t$ are omitted for brevity. 
	
	In this paper, we assume that the latency of computing the front $\ell_k$ layers of the DNN model on device $k$ follows the shifted exponential distribution\cite{shi2020joint,lee2017speeding}:
	\begin{align}\label{eqn:1}
		\mathbb{P}[\tau^{cp}_{\ell_{k}}<\theta] =\left\{
		\begin{aligned}
			&1-e^{-\frac{\epsilon_k}{c_{\ell_{k}}}(\theta-a_kc_{\ell_{k}})},&\  {\rm if\ }\theta\ge a_kc_{\ell_{k}}; \\
			&0,&\ {\rm otherwise.}\ \ \ 		
		\end{aligned}	
		\right. 	
	\end{align}
	$a_k$ and $\epsilon_k$ are fixed parameters indicating the maximum and fluctuation of the computation capability for device $k$, $c_{\ell_{k}}$ is the amount of computing of the front $\ell_k$ layers of the DNN model on device $k$, which is determined by the split point $\ell_k$. The product of $a_k$ and $c_{\ell_{k}}$ is the lower bound of the computing latency of device $k$, which corresponds to the computing capacity and load of device $k$. However, due to the complexity of the condition, the computing capacity of device is always difficult to reach the maximum, but will fluctuate within a certain range. The amount of computing load $c_{\ell_{k}}$ can be represented as the number of Multiply-Accumulate Operations (MACs), which is given by \cite{liu2018demand,xu2019reform}:
	\begin{align}\label{eqn:computation-layer}
		c_{\ell_{k}} =  \sum_{i = 1}^{\ell_k}\sum_{j=1}^{n_i}r^j_is^j_in_{i-1}h^j_iw^j_in_im\text{,}
	\end{align}
	where $r_i$ and $s_i$ are the $j$-th filter's kernel sizes of layer $i$, $n_i$ denotes the number of kernels in layer $i$, $h_i$ and $w_i$ are the corresponding height and width of output feature map respectively, and $m$ is the batch size. Accordingly, the data size of the intermediate result for each layer is
	\begin{align}\label{eqn:datasize}
		D_{\ell_k} = h_{\ell_k}w_{\ell_k} n_{\ell_k} m\text{.}
	\end{align}
	
	For communication aspect, the data transmission rate of device $k$ is expressed as
	\begin{align}\label{eqn:shannon}
		R_k = b_kW\log_2\left(1+\frac{p_k|g_k|^2}{N_0}\right)\text{,}
	\end{align}
	where $W$ is the total bandwidth, $b_k$ denotes bandwidth allocation ratio for device $k$, $p_k$ denotes the transmission power of device $k$, $g_k$ is the corresponding channel coefficient, and $N_0$ denotes the power of additive white Gaussian noise. 
	
	Therefore, according to (\ref{eqn:datasize}) and (\ref{eqn:shannon}), the transmission time for device $k$ can be denoted by:
	\begin{align}
		\tau^{cm}_{\ell_{k}} = \frac{D_{\ell_k}}{b_kW\log_2(1+\frac{p_k|g_k|^2}{N_0})}\text{.}
	\end{align}
	Due to the parallel training and transmission among the devices, the total latency of one forward and backward propagation of the system depends on the device with the largest latency, and the total latency is expressed by
	\begin{align}\label{eqn:total_t}
		\tau^{total} = \max_{k \in \mathcal{K}}\{\tau^{cp}_{\ell_{k}} + \tau^{cm}_{\ell_{k}}\}\text{,}
	\end{align}
	where $\mathcal{K} = \{1,2,...,K\}$ is the set of the devices.
	 
	In this paper, we adjust the split point of DNN and bandwidth allocation ratio to each device as to minimize the total latency in each device-edge iteration. The joint optimization problem of model splitting and bandwidth allocation can be formulated as follows:
	\begin{align}
			(\text{P1})\ \ \ \ \min_{\{\ell_k\},\{b_k\}} \quad &\mathbb{E}(\tau^{total}) \label{P1}\\
			{\rm\quad s.t.}\quad\ &{\rm Constraints\ \eqref{eqn:1}\ to\ \eqref{eqn:total_t},}\notag \\ 
			& \ell_k\in \{1,2,\cdots,L\},~\forall k, \tag{\ref{P1}{a}}\\ 
			& \ell_{k}(1-\frac{1}{K})\Phi^2\leq \hat{\Phi}^2,~\forall k, \tag{\ref{P1}{b}} \label{P1:c}\\ 
			&\sum_{k=1}^{K}b_k \leq 1,\  b_k\in [0,1],~\forall k. \tag{\ref{P1}{c}}
	\end{align}
	Here $\Phi^2=Z^2+\sigma^2$ which are constants defined in Assumption \ref{assumption1}, and $\hat{\Phi}^2$ denotes the limit on $\Phi^2$. Thus \eqref{P1:c} indicates the restriction on split point based on the convergence analysis. In another word, this is to guarantee a relatively small value of the split point so that a large part of the DNN model can be aggregated at the server for ensuring the accuracy performance.

	Note that solving problem P1 is challenging due to the integer nature of the split points. Therefore, in next section we decompose P1 into two sub-problems, solving the optimization of split points and bandwidth ratios respectively. However, searching the split point is tightly coupled with the bandwidth allocation for the goal of latency minimization, because both of these variables affect the transmission latency. Therefore, we use alternating optimization to solve the two variables.

	\section{Proposed Solution}
	\label{sec:model_splitting_and_bandwidth_allocation}
	In this section, we solve the problem P1 considered in the previous section. We decompose the original problem into two sub-problems with one for split point optimization and the other for bandwidth allocation, and we solve the two sub-problems by alternating optimization.
	\subsection{Optimal Model Splitting}\label{subsec:optimal_model_splitting}
	In this sub-section, we find the best split point of DNN for each device. For given the bandwidth ratios $\{b_k\}$, and let $\hat{L} = \left\lfloor\frac{\hat{\Phi}^2}{(1-\frac{1}{K})\Phi^2}\right\rfloor$, here we have the following optimization problem:
	\begin{align}
		(\text{P2})\ \ \ \ \min_{\{\ell_k\}}& \quad \mathbb{E}(\tau^{total}) \label{P2}\\
		{\rm s.t.}& \quad {\rm Constraints\ \eqref{eqn:1},\eqref{eqn:computation-layer},\eqref{eqn:total_t},}\notag\\
		&\quad \ell_{k}\leq \hat{L},~\forall k, \tag{\ref{P2}{a}} \label{P2:a}\\
		&\quad \ell_k\in \{1,2,...,L\},~\forall k \tag{\ref{P2}{b}}.
	\end{align}
	Firstly, it is readily observed that P2 can be decoupled to $K$ sub-problems, where each sub-problem corresponds to one device and can be solved independently. Therefore, for each independent sub-problem of solving $\ell_k$, we use the method of backward induction to decide the optimal splitting strategy for each layer. Because the choice of the split point is limited by the $\hat{L}$, it starts from the deepest layer of DNN available, i.e., layer $\hat{L}$, and calculates the expected latency $\mathbb{E}(V_{\hat{L}_k})$. Then the $V_{(\hat{L}-1)_k}$ and computing latency threshold $\hat{\tau}^{cp}_{(\hat{L}-1)_k}$ at layer $\hat{L}-1$ are calculated according to the result of layer $\hat{L}$. The method continues to calculate forward in turn until the first layer of DNN. In this way, we can calculate the threshold $\hat{\tau}^{cp}_{\ell_{k}}$ and expected latency at each layer. The splitting strategy is that the devices execute the calculation of each layer in sequence and records $\tau^{cp}_{\ell_{k}}$. If $\tau^{cp}_{\ell_{k}}<\hat{\tau}^{cp}_{\ell_{k}}$, it will be split at this layer, the local calculation of device will end and the intermediate result will be sent to server. Otherwise, device continues the calculation of the next layer of DNN and repeats the same step until layer $\hat{L}$. 
	
	We use $V_{\ell_{k}}$ to represent the minimum expected total latency for device $k$ of splitting at layer $\ell_k$. According to the actual computing latency $\tau^{cp}_{\ell_{k}}$, in layer $\ell_k$ we have:
	\begin{align}
		\label{eqn:bacakward_induction}
		\begin{split}
			V_{\ell_{k}} &= \min\{\tau^{cp}_{\ell_{k}}+\tau^{cm}_{\ell_{k}},\ \mathbb{E}(V_{(\ell+1)_{k}})\}\\
			&=	\left\{\begin{array}{l@{\quad}l}
				\tau^{cp}_{\ell_{k}}+\tau^{cm}_{\ell_{k}},&\ {\rm if}\ \tau^{cp}_{\ell_{k}}<\hat{\tau}^{cp}_{\ell_{k}},\\
				\mathbb{E}(V_{(\ell+1)_{k}})  ,&\ \rm{otherwise,}
			\end{array}
			\right. 
		\end{split}
	\end{align}
	where 
	\begin{align}\label{eqn:backward threshold}
		\hat{\tau}^{cp}_{\ell_{k}} = \mathbb{E}(V_{(\ell+1)_{k}})-\tau^{cm}_{\ell_{k}},
	\end{align}
	which indicates that the threshold determines whether to split in layer $\ell_k$. In layer $\hat{L}$, the expectation of latency can be calculated by:
	\begin{align}
		\begin{aligned}
			\mathbb{E}(V_{\hat{L}_k})&=\int_{a_kc_{\hat{L}_k}}^{+\infty}\frac{\epsilon_k}{c_{\hat{L}_k}}\theta \exp\big\{-\frac{\epsilon_k}{c_{\hat{L}_k}}(\theta-a_kc_{\hat{L}_k})\big\}\mathrm{d}\theta + \tau^{cm}_{\hat{L}_k}\\
			&=\left(a_k+\frac{1}{\epsilon_k}\right)c_{\hat{L}_k}+\tau^{cm}_{\hat{L}_k}.
		\end{aligned}
	\end{align}
	
	Then we continue backward induction, in layer $\ell_k$: 
	\begin{align}
		\begin{aligned}
			\mathbb{E}(V_{\ell_{k}}) &= \mathbb{E}\left(\min\{\tau^{cp}_{\ell_{k}}+\tau^{cm}_{\ell_{k}},\ \mathbb{E}(V_{(\ell+1)_{k}})\}\right)\\
			&= \int_{a_kc_{\ell_{k}}}^{\hat{\tau}^{cp}_{\ell_{k}}}(\theta+\tau^{cm}_{\ell_{k}})\frac{\epsilon_k}{c_{\ell_{k}}}\exp\big\{-\frac{\epsilon_k}{c_{\ell_{k}}}(\theta-a_kc_{\ell_{k}})\big\}\mathrm{d}\theta \\&\quad+ \int_{\hat{\tau}^{cp}_{\ell_{k}}}^{+\infty}\mathbb{E}(V_{(\ell+1)_{k}})\exp\big\{-\frac{\epsilon_k}{c_{\ell_{k}}}(\theta-a_kc_{\ell_{k}})\big\}\mathrm{d}\theta\\
			&= a_kc_{\ell_{k}}+\frac{c_{\ell_{k}}}{\epsilon_k}-(\hat{\tau}^{cp}_{\ell_{k}}+\frac{c_{\ell_{k}}}{\epsilon_k})\exp\big\{{-\frac{\epsilon_k}{c_{\ell_{k}}}(\hat{\tau}^{cp}_{\ell_{k}}-a_kc_{\ell_{k}})}\big\}\\&\quad+\Big(1-\exp\big\{{-\frac{\epsilon_k}{c_{\ell_k}}(\hat{\tau}^{cp}_{\ell_{k}}-a_kc_{\ell_{k}})}\big\}\Big)\tau_{\ell_{k}}^{cm} \\ 
			&\quad+\mathbb{E}(V_{(\ell+1)_{k}})\exp\big\{{-\frac{\epsilon_k}{c_{\ell_{k}}}(\hat{\tau}^{cp}_{\ell_{k}}-a_kc_{\ell_{k}})}\big\}.
		\end{aligned}
	\end{align}
	 Because the training process is not suitable for the dynamic change of split points, we fix at an appropriate split point for each device during the whole training process. The probability of model splitting in each layer for device $k$ is
	\begin{align}
		\label{eqn:split_prob}
		\begin{split}
			\mathbb{P}(\ell_k) = \left\{ \begin{array}{l@{\ }l}
				1-\exp\big\{{-\frac{\epsilon_k}{c_{1_k}}(\hat{\tau}^{cp}_{1_k}-a_kc_{1_k})}\big\},&\ell_k=1;\\
				\Big(\prod_{j_k=1}^{\ell_k-1}\exp\big\{{-\frac{\epsilon_k}{c_{j_k}}(\hat{\tau}^{cp}_{j_k}-a_kc_{j_k})}\big\}\Big)\\
				\times\Big(1-\exp\big\{{-\frac{\epsilon_k}{c_{\ell_{k}}}(\hat{\tau}^{cp}_{\ell_{k}}-a_kc_{\ell_{k}})}\big\}\Big),&1<\ell_k<\hat{L};\\
				\prod_{j_k=1}^{\hat{L}-1}\exp\big\{{-\frac{\epsilon_k}{c_{j_k}}(\hat{\tau}^{cp}_{j_k}-a_kc_{j_k})}\big\},&\ell_k=\hat{L}.
			\end{array}\right. 
		\end{split}
	\end{align}
	In this way, we select the split point at the layer with the maximum $\mathbb{P}(\ell_k)$:
	\begin{align}
		\label{eqn:opt_splayer}
		\ell^*_{k} =  \ \arg\ \max_{\ell_k}\{ \mathbb{P}(\ell_k)\}.
	\end{align} 
	
	Using \eqref{eqn:split_prob}, the expectation of the computing latency can be obtained by
	\begin{align}
		\begin{aligned}	
		\label{eqn:expect_splitpoint}
		\mathbb{E}(\tau^{cp}_{\ell_{k}}|\ell_k=\ell^*_k) =& \int_{a_kc_{\ell_{k}^*}}^{\infty}\theta\frac{\epsilon_k}{c_{\ell_{k}^*}}\exp\big\{-\frac{\epsilon_k}{c_{\ell_{k}^*}}(\theta-a_kc_{\ell_{k}^*})\big\}\mathrm{d}\theta\\
		=& c_{\ell_{k}^*}\Big(a_k+\frac{1}{\epsilon_k}\Big).
	\end{aligned}
	\end{align}
	
	
	\subsection{Bandwidth Allocation}
	The computing and communication conditions of devices are different, so the latency of all devices can vary widely. Here we allocate bandwidth for each device to minimize the expected latency of the whole system for given optimal split points ${\ell_k^*}$. The sub-problem is:
	
	\begin{align}
		\begin{aligned}
			\text{(P3)}\quad\min_{\{b_k\}}&\quad \mathbb{E}(\tau^{total})\\
			{\rm s.t.}&\quad{\rm Constraints\ \eqref{eqn:datasize}\ to\ \eqref{eqn:total_t},}\\
			&\quad \sum_{k=1}^{K}b_k \leq 1,\  b_k\in [0,1],~\forall k.
		\end{aligned}
	\end{align}
	
	Note that the optimal solution of problem P3 can be established if and only if the total delay of all devices is equal, because the total latency of the system is limited by the worst device due to the parallel computing and communication among devices, see \eqref{eqn:total_t}. This can be readily proved by contradiction: if the latency of a device is longer than that of another device, part of the bandwidth of the device with shorter latency can be allocated to the former, so as to reduce the latency. Therefore, when the latency of all devices is equal, it is optimal. Assume that the optimal latency is $\tau_k^*$, and we have the following theorem:
	\begin{theorem}
		\label{theorem1}
		The solution of P3 is given as:
		\begin{align}
			\label{eqn:bw_ration}		
			b_k^* = \frac{D_{\ell^*_k}}{[\tau_k^*-\tau_{\ell_{k}^*}^{cp}]W\log_2(1+\frac{p_k|g_k|^2}{N_0})},
		\end{align}
		where $D_{\ell^*_k}$ and $\tau_{\ell_{k}^*}^{cp}$ denote the intermediate data size corresponding to the split point and expected computing latency obtained by P2, respectively.
	\end{theorem} 
	\begin{proof}
		See Appendix B. 
	\end{proof}
	In this way, solving P3 only needs to find $\tau_k^*$ in \eqref{eqn:bw_ration}, but it is difficult to give a close-form of it, so we use the binary search method to find the optimal value of $\tau_k^*$ as shown in Alg. 1. 
	
	
	\begin{algorithm}
		\renewcommand{\algorithmicrequire}{\textbf{Input:}}
		\renewcommand{\algorithmicensure}{\textbf{Output:}}
		\caption{Binary Search for Bandwidth Allocation}
		\label{alg:1}
		\begin{algorithmic}[1]
			\REQUIRE Optimal computing latency for each device $\tau_{\ell_{k}^*}^{cp}$, transmission data size $D_{\ell_k^*}$, total bandwidth $W$.
			\ENSURE Bandwidth allocation ratio $\{b_k\}$, optimal latency $\tau_k^*$.
			\STATE $\textbf{initialize:}$ $\tau_{low}=\max_{k \in \mathcal{K}}\{\tau_{\ell_{k}^*}^{cp}\}$, a big enough $\tau_{up}$, allowable error $\varepsilon$, $\tau=\tau_{up}$.
			\WHILE{$\sum_{k=1}^K b_k<(1-\varepsilon)$ OR $\sum_{k=1}^K b_k>1$}
			\STATE Substituting $\tau_k^*$ by $\tau$, calculate $b_k$ for all devices $k\in \mathcal{K}$ with \eqref{eqn:bw_ration};
			\STATE Compute the total bandwidth allocation ration $\sum_{k=1}^K b_k$;
			\IF{$0<\sum_{k=1}^K b_k<1-\varepsilon$}
			\STATE Adjust the searching region: $\tau_{up} = \tau,\ \tau=\frac{\tau+\tau_{low}}{2}$;
			\ELSIF{$\sum_{k=1}^K b_k>1$}
			\STATE Adjust the searching region: $\tau_{low} = \tau,\ \tau=\frac{\tau+\tau_{up}}{2}$;
			\ELSIF{$1-\varepsilon<\sum_{k=1}^K b_k<1$}
			\STATE Obtain the solution with $\varepsilon$, break loop.
			\ENDIF
			\ENDWHILE
			\STATE \textbf{return} $\{b_k\} $ and $\tau_k^*$.
		\end{algorithmic}
	\end{algorithm}
	\subsection{Alternating Optimization and Complexity}
	In the previous two sub-sections, we solve one of the two sub-problems P2 and P3 by assuming the other to be fixed. Note that the change of the split points will result in a different amount of data to be sent, which in turn affects the required bandwidth. Similarly, the allocated bandwidth will also affect the optimal split points. Therefore, P2 and P3 are solved by alternating optimization as shown in Alg. \ref{alg:2}, in which the output of P2 (P3) is the input of P3 (P2), and the process is repeated until  converge or maximum number of alternation $N_{Iter}$ is reached.
	
	In the solution of P2, the threshold $\hat{\tau}^{cp}_{\ell_k}$ of device $k$ need to be calculated for each layer, therefore, the computation complexity of P2 is $\mathcal{O}(K\hat{L})$. In the solution of P3, because the allocation ratios of all devices needs to be calculated in each loop, the computation complexity of P3 is $\mathcal{O}(K\log_2(\frac{\tau_{up}}{\varepsilon}))$. Finally, the alternating optimization in Alg. 2 is executed for $N_{Iter}$ times, thus the total computation complexity of the whole algorithm is $\mathcal{O}(N_{Iter}K\hat{L})$.
	
	
	\begin{algorithm}
		\caption{Alternating Optimization for P1}
		\label{alg:2}
		\begin{algorithmic}[1]
			\STATE \textbf{Initialize:} $\{b_k\}$ and $N_{Iter}$. 
			\FOR{$i  = 1 : N_{Iter}$}
			\STATE Find optimal split point $\ell_k^*$ using \eqref{eqn:opt_splayer} for given $\{b_k\}$;
			\STATE Find optimal ratio $\{b_k^*\}$ using \eqref{eqn:bw_ration} for given $\{\ell_k\}$;
			\ENDFOR
		\end{algorithmic}
	\end{algorithm}

	\section{Experimental Results}
	\label{sec:experiment}
	In this section, we evaluate the performance of the proposed algorithm for SFL.
	\subsection{Experiment Settings}
	Unless otherwise stated, we set up 20 devices to participate in the SFL. In order to evaluate the performance conveniently, we assume that each device equips a CPU with maximum frequency in range of $[1,5]$ GHz, and can process one MAC operation in each CPU cycle. Therefore, we set $a_k\in [0.2\times10^{-9},1\times10^{-9}]$ s/MAC randomly and set $\epsilon_k=\frac{2}{a_k}$ for the computing latency model. The total bandwidth for the system is $W=20$ MHz, and the transmit power of device is set to be $p_k=10$ dBm. We set the power spectrum density of the additive Gaussian noise to be $N_0=-114$ dBm, and the channel power gain is modeled as $|g_k|^2=\rho|\tilde{g}|^2$, where $|\tilde{g}|^2$ denotes small-scale fading components that follows normalized exponential distribution, and $\rho$ is the large-scale fading components given by $128.1+37.6\log_{10}(10^{-3}d_k)$ dB, in which $d_k$ is the distance between server and device $k$. We use AlexNet\cite{NIPS2012_c399862d} as the DNN model in the ML experiment, which contains 5 convolutional layers and 3 fully connected layers. The AlexNet structure  used in this paper is shown as Fig. \ref{fig:AlexNet}.
	\begin{figure}
		\begin{centering}
			\includegraphics[width=1 \linewidth]{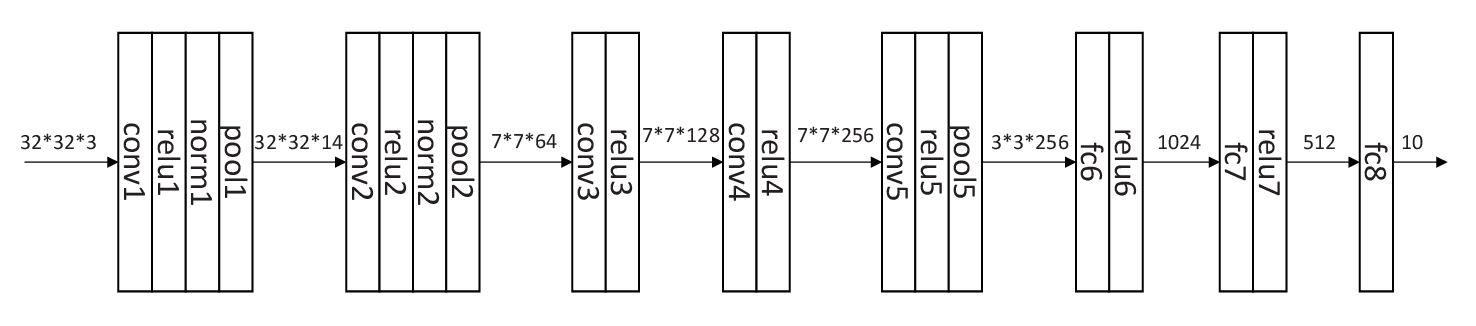}
			\vspace{-0.4cm}
			\caption{The structure of AlexNet.}\label{fig:AlexNet}
		\end{centering}
	\vspace{-1.2cm}
	\end{figure}
	\subsection{Computation Amount of DNN Model}
	\begin{figure}[ht]
		\centering
		\includegraphics[width=1 \linewidth, trim=8cm 0.1cm 6cm 0.6cm, clip]{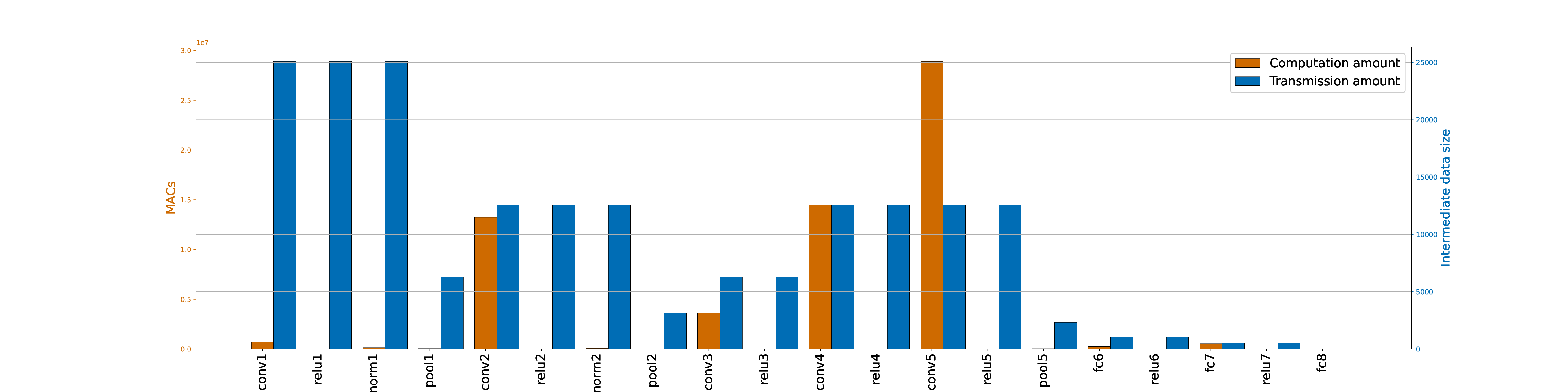}
		\vspace{-0.4cm}
		\caption{ Layer amount of calculation and transmitting data for AlexNet.}\label{fig:layer_data}
		\vspace{-0cm}
	\end{figure}
	We evaluate the latency performance of our proposed system with different DNN models by experiment. In our experiment, AlexNet is considered as structures with 20 layers. We first calculate the amount of MACs and transmitting data of each layer for AlexNet, as showed in Fig. \ref{fig:layer_data}. We can observe that the Convolutional layers have the most amount of computation, which occupies more than 90\% of the total amount. In contrast, the computation amount of Activation layers (relu) and Normalization layers (norm) is much lower. On the other hand, it is noted that the amount of data transmitted shows a downward trend with the increase of layer number.

	\subsection{Total Latency Performance}
	
	We study the influence of split points on different DNN models. 	
	Fig. \ref{fig:ALexNettt} and Fig. \ref{fig:VGG1666} show the optimal split point and expected latency of AlexNet and VGG16, respectively. $\hat{L}$ is set to be 8 for AlexNet and 6 for VGG16. We can observe that as the average device-server distance becomes small, the split point being closer to the input side. This is because  that short distance means high transmission rate and thus the devices prefer transmitting more data to server. In addition, the splitting point is closer to the output side when the computing capacity of devices (i.e., $a_k$) increases, because the higher layers have smaller amount of transmission data as Fig. \ref{fig:layer_data} shows, which benefits reducing the transmitting latency. In Fig. \ref{fig:SNRvslatency}, we illustrate the total average latency $\tau^{total}$ v.s. average device-server distance. Obviously, lower $a_k$ and shorter distance, i.e., better computing and communication conditions, have benefits in reducing total latency. 

	\begin{figure}[ht]
	\centering
	\subfigure[Optimal splitting points.]{
		\centering
		\includegraphics[width=0.44 \linewidth, trim=0.5cm 0.1cm 1.1cm 0.7cm, clip]{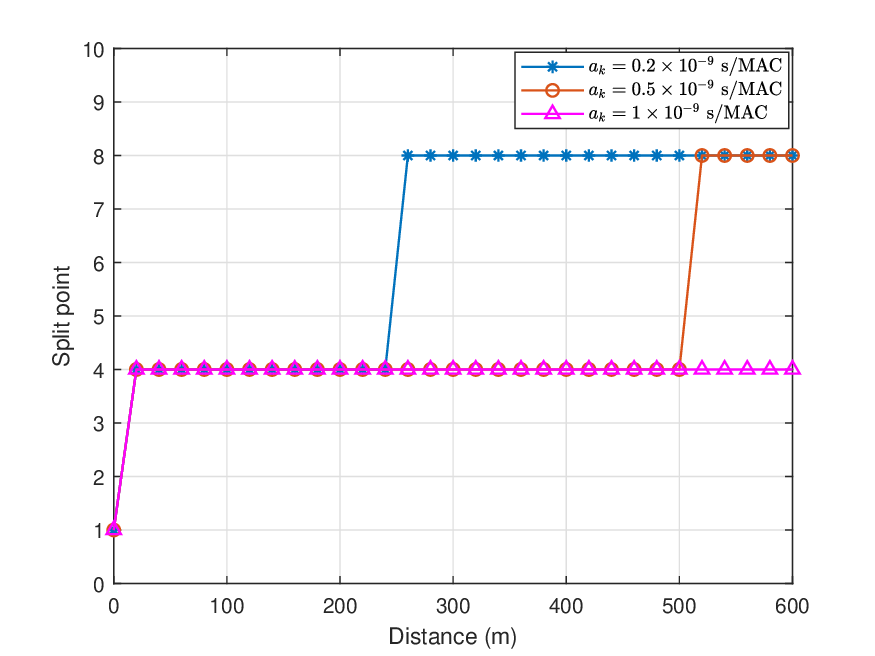}
		\vspace{-0.8cm}
		\label{fig:SNRvsEl}
	}%
    \ 
	\subfigure[Expected latency.]{
		\centering
		\includegraphics[width=0.46 \linewidth, trim=0.4cm 0.1cm 0.9cm 0.7cm, clip]{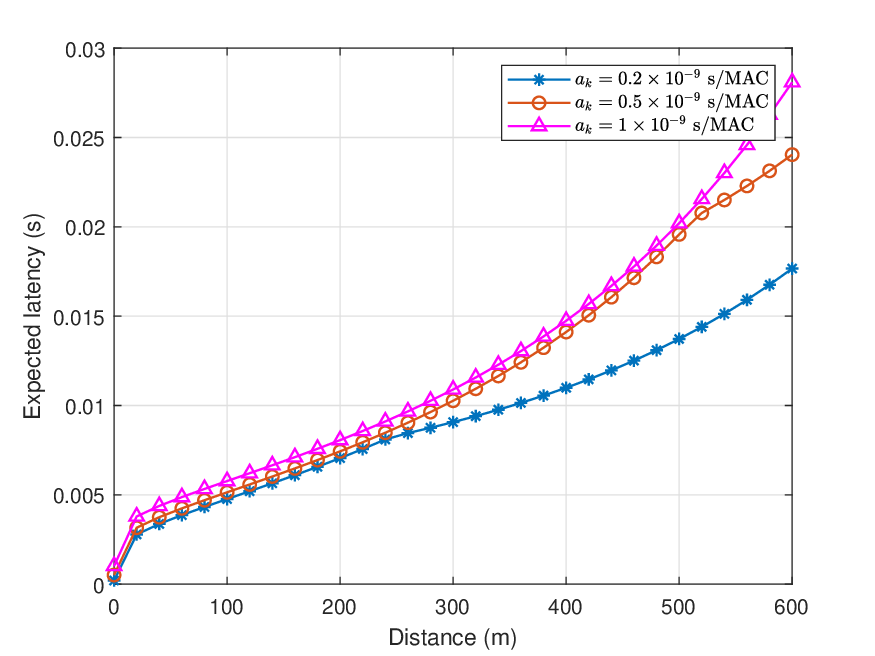}
		\vspace{-0.8cm}
		\label{fig:SNRvslatency}
	}%
	\centering
	\vspace{-0.2cm}
	\caption{Split point and expected latency v.s. distance with AlexNet.}
	\label{fig:ALexNettt}
	\vspace{-0.2cm}
\end{figure}
	\begin{figure}[ht]
		\centering
		\subfigure[Optimal splitting points.]{
			\centering
			\includegraphics[width=0.44 \linewidth, trim=0.5cm 0.1cm 1.1cm 0.7cm, clip]{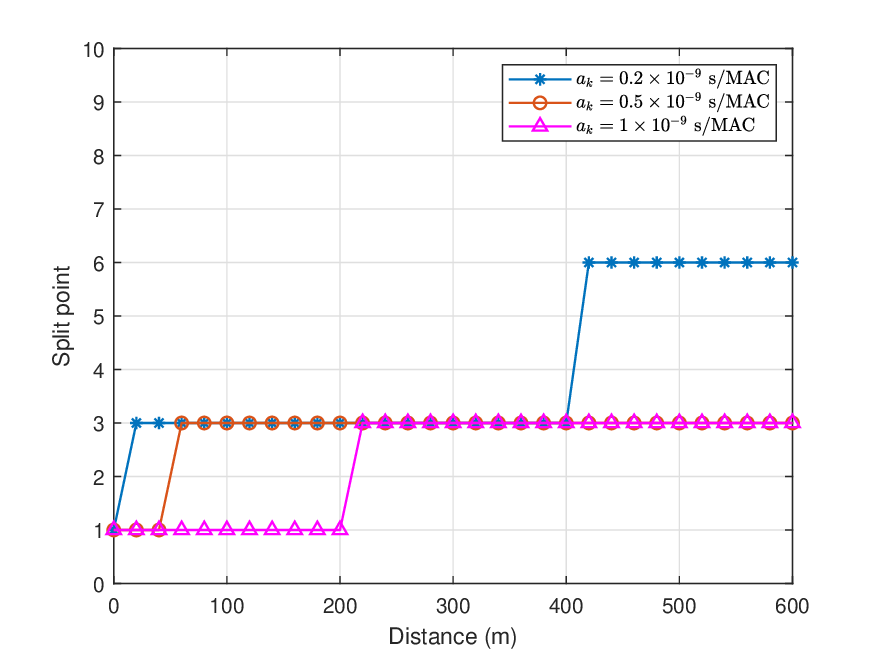}
			\vspace{-0.8cm}
			\label{fig:SNRvsEl(VGG16)}
		}%
		\ 
		\subfigure[Expected latency.]{
			\centering
			\includegraphics[width=0.46 \linewidth, trim=0.4cm 0.1cm 0.9cm 0.7cm, clip]{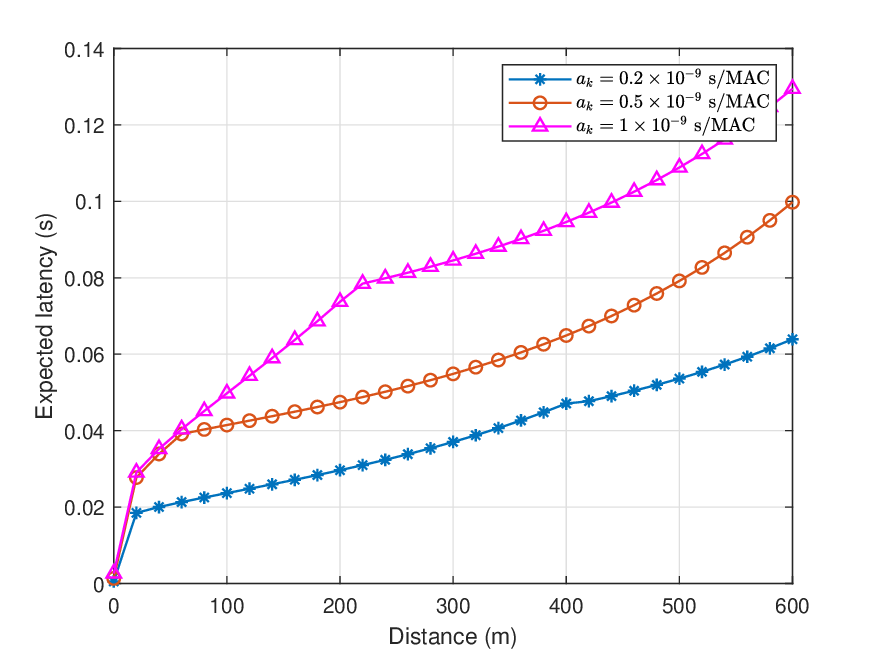}
			\vspace{-0.8cm}
			\label{fig:SNRvslatency(VGG16)}
		}%
		\centering
		\vspace{-0.2cm}
		\caption{Split point and expected latency v.s. distance with VGG16.}
			\label{fig:VGG1666}
		\vspace{-1.2cm}
	\end{figure}

%

	\subsection{Impact of the upper bound of split points}
	In the optimization of split points, $\hat{L}$ is a key parameter that trade-offs latency and training performance. Therefore, in this section we evaluate the impact of $\hat{L}$. Fig. \ref{split_upper_bound} illustrates the effect of $\hat{L}$ on split point optimization, where the horizontal coordinate is the computational frequency of the device. It can be seen that the split point selection is the same until the upper limit is reached for different $\hat{L}$. Therefore, the setting of $\hat{L}$ has no effect on the split point optimization before the upper limit is reached. However, due to $\hat{L}$, the split point cannot increase further after reaching the upper limit.
	
	\begin{figure}[h]
		\centering
		\subfigure[Optimal split points for AlexNet.]{
			\includegraphics[width=0.46 \linewidth, trim=0.8cm 0.1cm 0.9cm 0.7cm, clip]{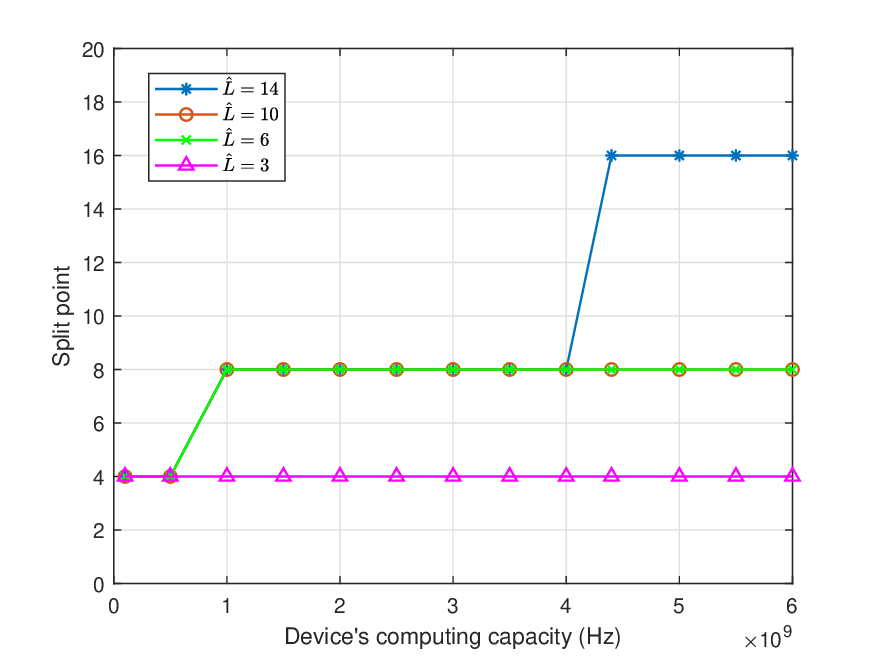}
			\vspace{-0.8cm}
		}
		\ 
		\subfigure[Optimal split points for VGG16.]{
			\includegraphics[width=0.46 \linewidth, trim=0.8cm 0.1cm 0.9cm 0.7cm, clip]{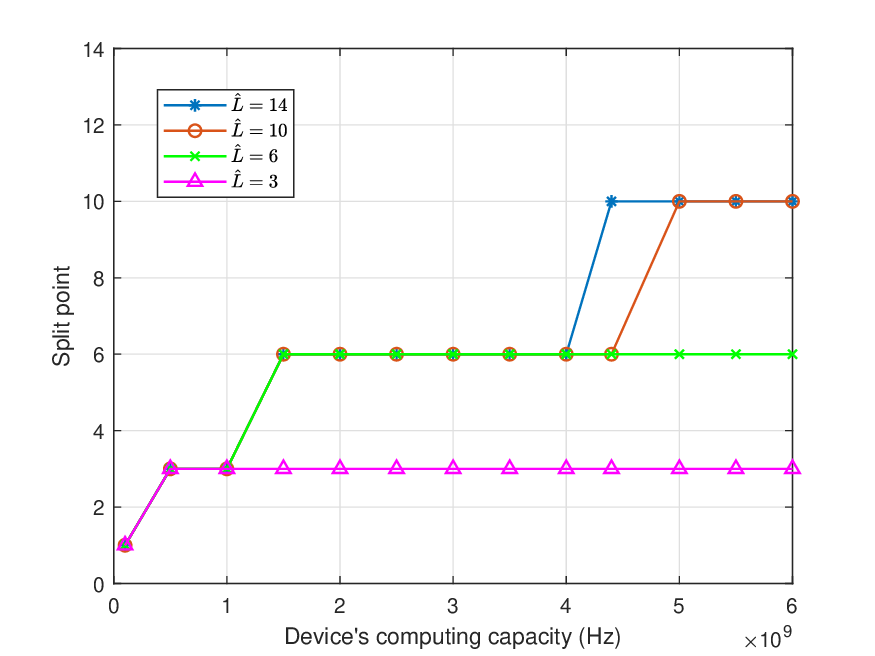}
			\vspace{-0.8cm}
		}	
		\caption{Optimal split points for different $\hat{L}$.}
		\label{split_upper_bound}
	\end{figure}

	\subsection{SFL v.s. FL}
	\begin{figure}[ht]
	\centering
	\subfigure[Time consumption for AlexNet.]{
			\centering
			\includegraphics[width=0.48 \linewidth, trim=0.1cm 0.1cm 1cm 0.6cm, clip]{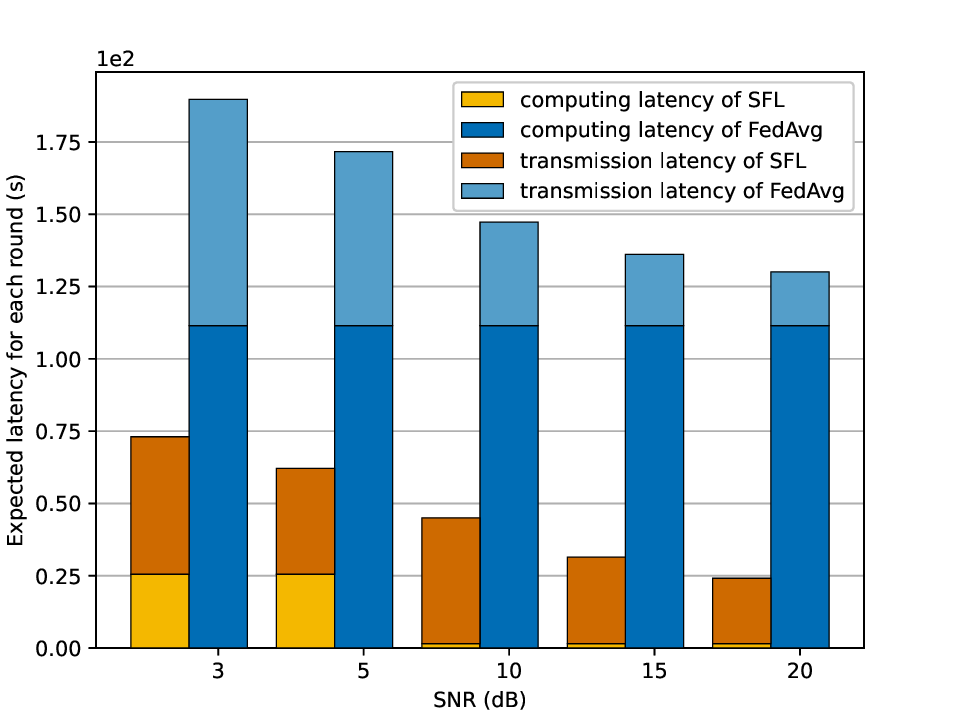}
			\vspace{-0.8cm}
			\label{fig:time_compute}
	}%
	\subfigure[Time consumption for VGG16.]{
			\centering
			\includegraphics[width=0.48 \linewidth, trim=0.1cm 0.1cm 1cm 0.6cm, clip]{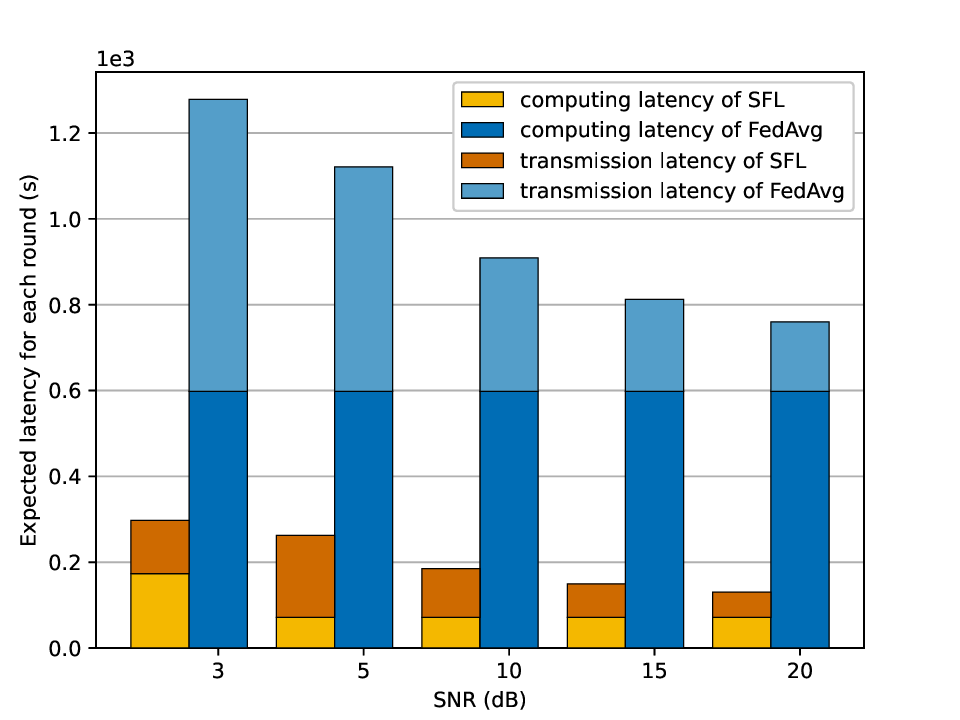}
			\vspace{-0.8cm}
			\label{fig:time_compute(VGG16)}
	}%
	\vspace{-0.2cm}
	\caption{Latency comparison between the proposed SFL and FedAvg under different communication conditions.}
		\label{fig:result_time}
	\vspace{0.6cm}

	\end{figure}	
	\begin{figure}[]
		\vspace{-0.8cm}
		\centering
		\subfigure[Time consumption for AlexNet.]{
				\centering
				\includegraphics[width=0.48 \linewidth, trim=0.1cm 0.1cm 1cm 0.6cm, clip]{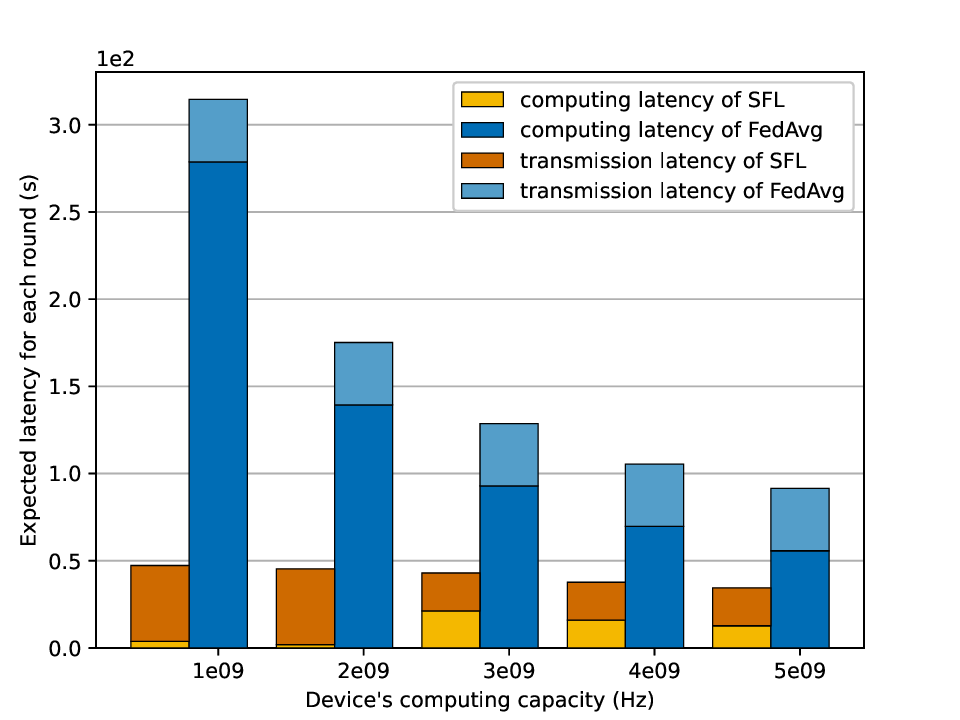}
				\vspace{-0.8cm}
				\label{fig:time_compute(AlexNet-cpu)}
		}%
		\subfigure[Time consumption for VGG16.]{
				\centering
				\includegraphics[width=0.48 \linewidth, trim=0.1cm 0.1cm 1cm 0.6cm, clip]{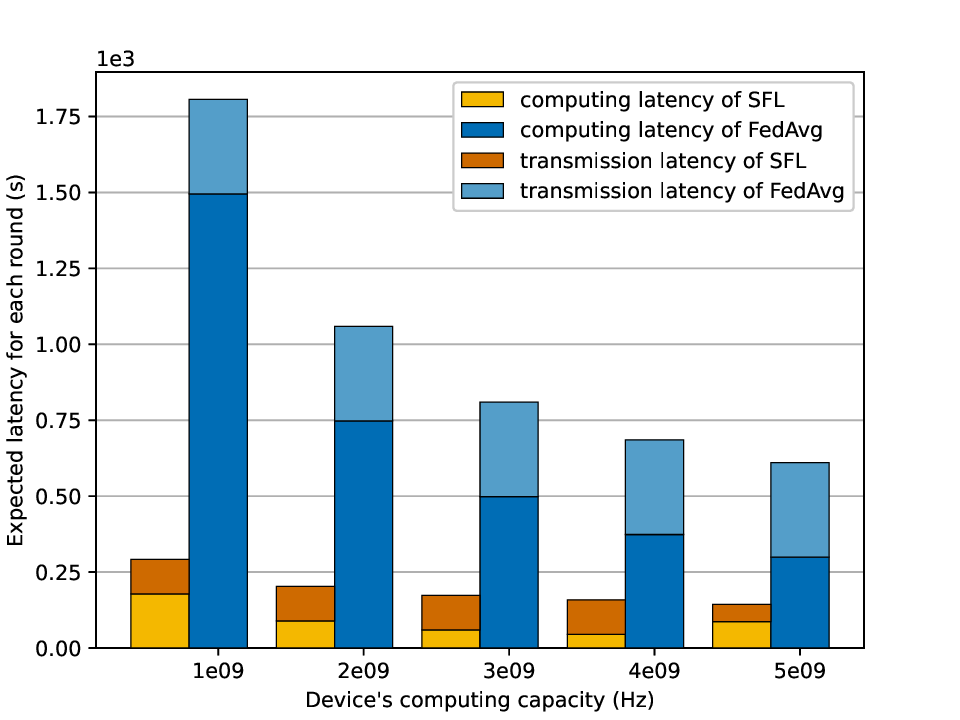}
				\vspace{-0.8cm}
				\label{fig:time_compute(VGG16-cpu)}
		}%
		\quad\quad
		\vspace{-0.2cm}
		\caption{Latency comparison between the proposed SFL and FedAvg under different computing conditions.}
		\label{fig:result_time_cpu}
		\vspace{-1.2cm}
	
	\end{figure}

	In order to test the effect of the proposed scheme, we compare the latency performance of the proposed SFL scheme and FedAvg, where Fig. \ref{fig:result_time} and Fig. \ref{fig:result_time_cpu} show the comparison under communication and computation conditions, respectively. We follow the previous experimental settings and assume that there are 1,500 training objects locally for each device. It can be seen that the computing latency of the proposed SFL has a huge gain over the classical FedAvg. For VGG16, both of the computing and transmitting latency are much lower than that of FedAvg, which saves more than 75\% latency. For AlexNet, although the communication latency of the proposed SFL scheme is a little higher than that of FedAvg in the case of better channels, the computing latency is greatly reduced, which makes the overall latency significantly lower than that of FedAvg.

	\subsection{Test Accuracy of DNN}
	
	We evaluated the training performance of our work in different situations under the MNIST dataset for handwritten digits classification and CIFAR10 for image classification. The MNIST dataset has 60,000 training images and 10,000 testing images of 10 digits and CIFAR10 has 50,000 training images and 10,000 testing images. We consider the data distribution of both IID and non-IID. In IID setting, each device is randomly assigned to the same amount of data, and each type of data has the same probability assigned to device, while in non-IID setting each device is assigned the same amount of data but the distribution of labels is uneven. 
	First, we compare the training results of different number of devices $K$, as shown in Fig. \ref{fig:numacc}, where the split point is set to be $\ell_k=4$.

	\begin{figure}[]
		\vspace{-0.1cm}
		\centering
		\subfigure[MNIST, $K=5$.]{
			\includegraphics[width=0.46 \linewidth, trim=0.1cm 0.1cm 1cm 0.6cm, clip]{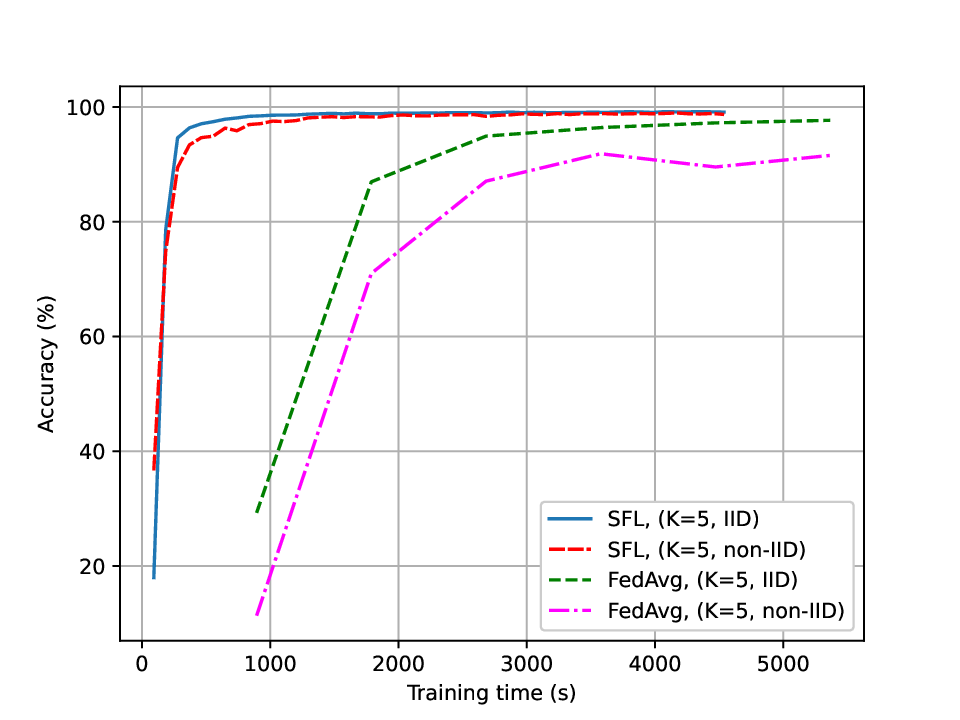}
			\vspace{-1.5cm}
		}
		\ 
		\subfigure[MNIST, $K=20$.]{
			\includegraphics[width=0.46 \linewidth, trim=0.1cm 0.1cm 1cm 0.6cm, clip]{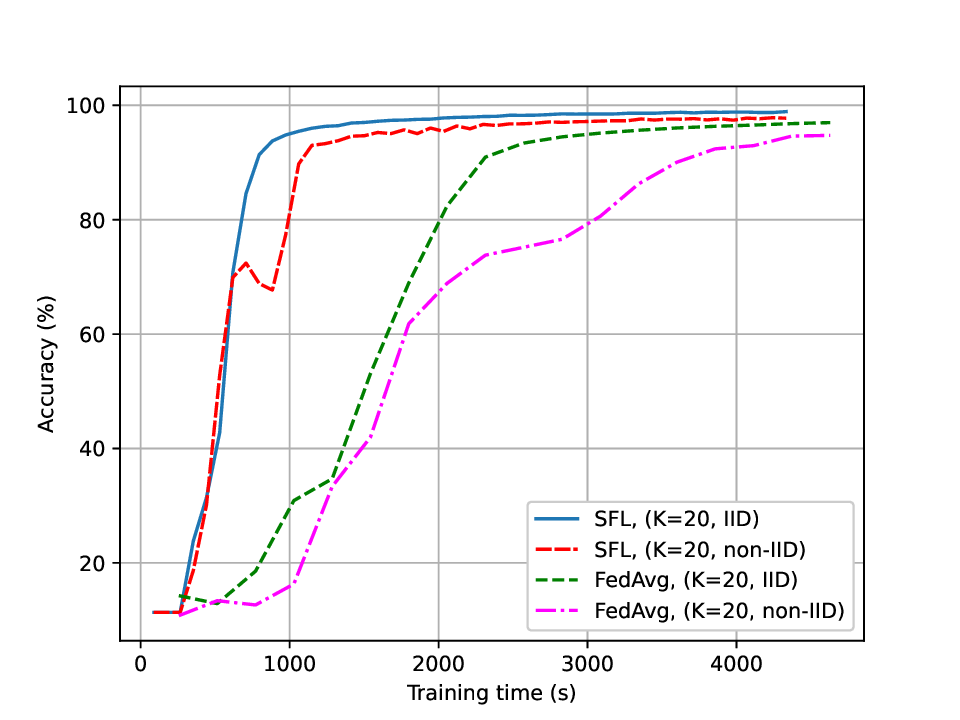}
			\vspace{-1cm}
		}
		\ 
		\subfigure[MNIST, $K=100$.]{
			\includegraphics[width=0.46 \linewidth, trim=0.1cm 0.1cm 1cm 0.6cm, clip]{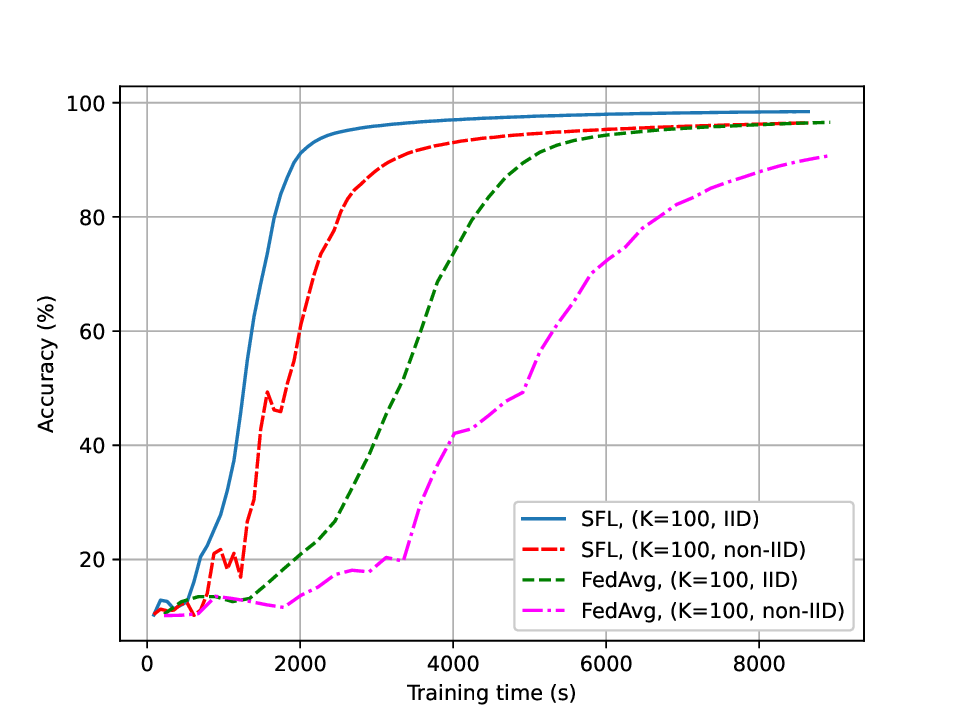}
			\vspace{-1.5cm}
		}
		\ 
		\subfigure[CIFAR10, $K=5$.]{
			\includegraphics[width=0.46 \linewidth, trim=0.1cm 0.1cm 1cm 0.6cm, clip]{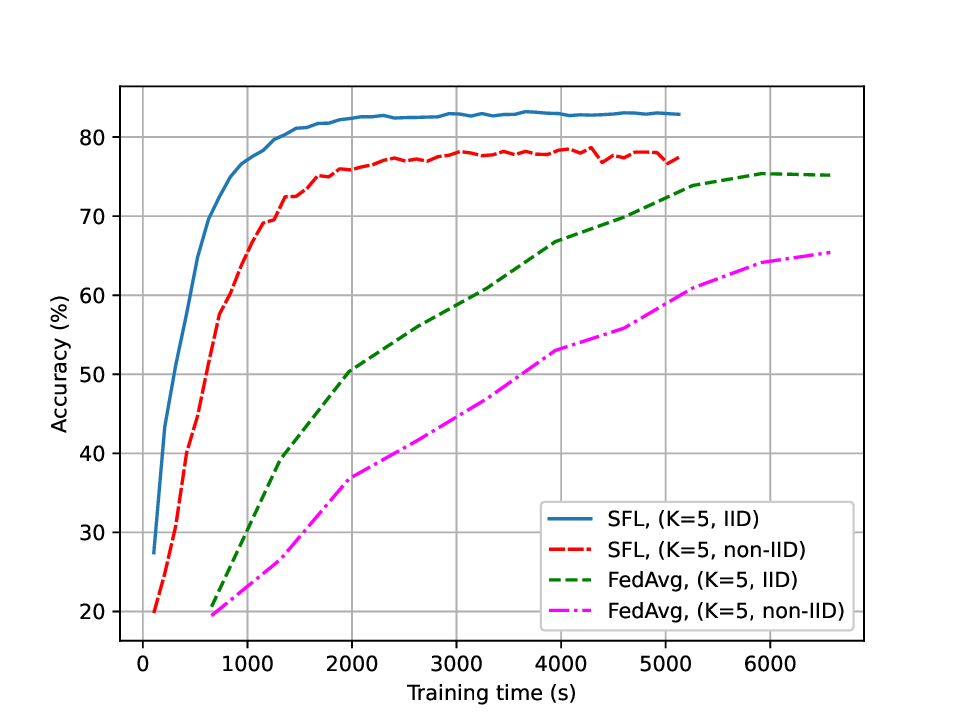}
			\vspace{-1.5cm}
		}
		\  
		\subfigure[CIFAR10, $K=20$.]{
			\includegraphics[width=0.46 \linewidth, trim=0.1cm 0.1cm 1cm 0.6cm, clip]{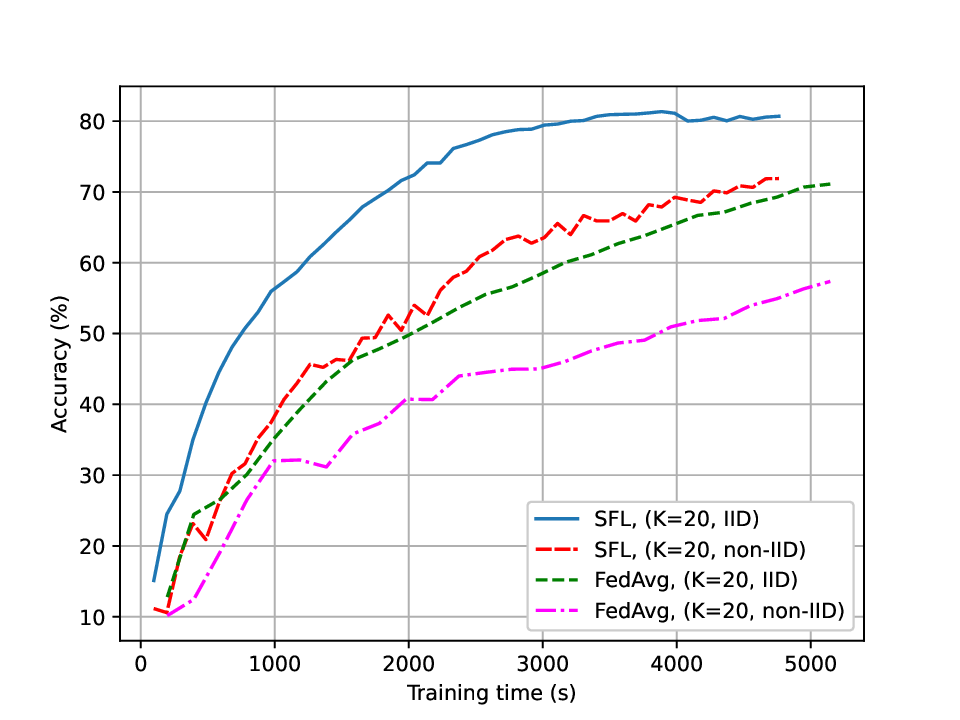}
			\vspace{-1.5cm}
		}
		\ 
		\subfigure[CIFAR10, $K=100$.]{
			\includegraphics[width=0.46 \linewidth, trim=0.1cm 0.1cm 1cm 0.6cm, clip]{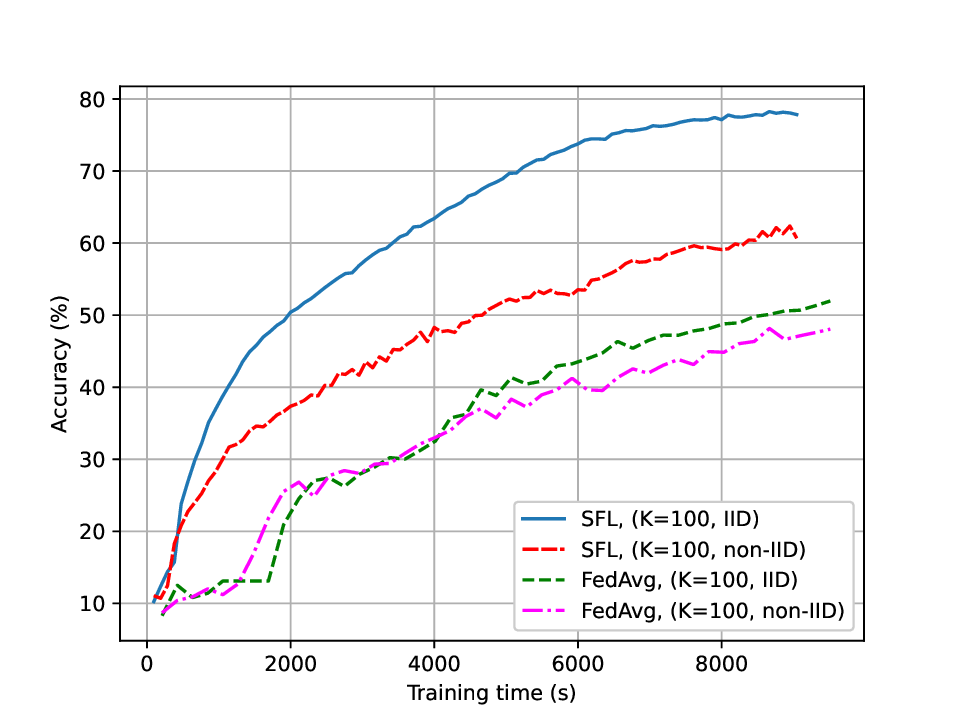}
			\vspace{-1.5cm}
		}
		\caption{Testing accuracy for different dataset and number of devices.}	
		\label{fig:numacc}
		\vspace{-0cm}
		
	\end{figure}	
	
	
	From the results we have the following observation. It is obvious that SFL can achieve higher accuracy than FL in the same time, because the latency of SFL per round is significantly lower than FL. Besides, the increase of the number of devices leads to a decrease in testing performance. 
	\begin{figure}[]
		\centering
		\subfigure[Testing accuracy for MNIST, IID.]{
			\includegraphics[width=0.46 \linewidth, trim=0.2cm 0.1cm 1cm 0.6cm, clip]{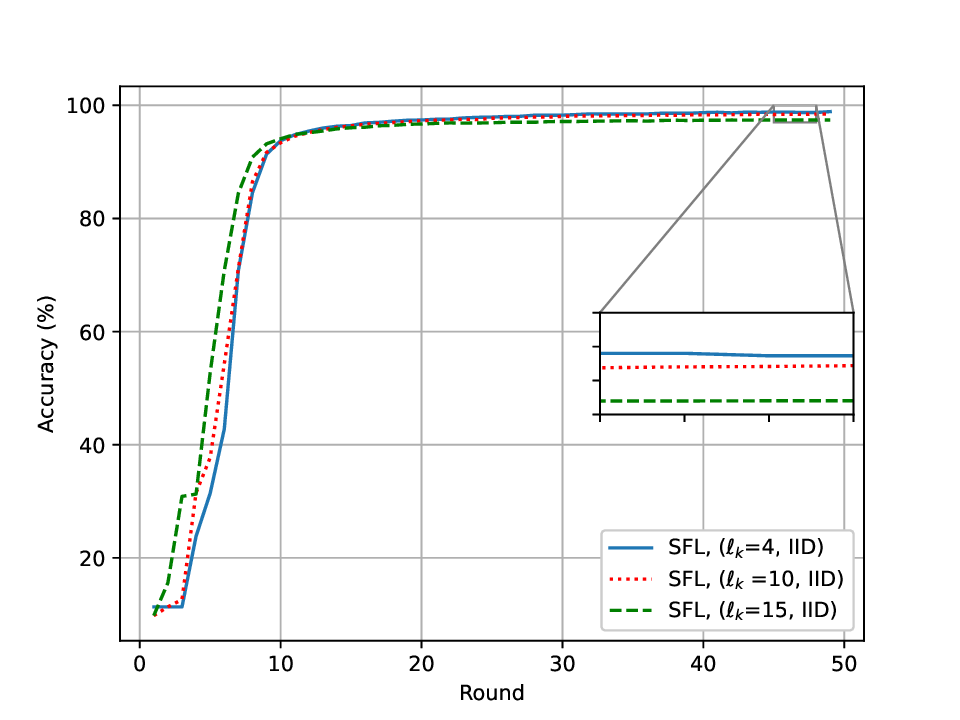}
			\vspace{-0.2cm}
		}
		\subfigure[Testing accuracy for MNIST, non-IID.]{
			\includegraphics[width=0.46 \linewidth, trim=0.2cm 0.1cm 1cm 0.6cm, clip]{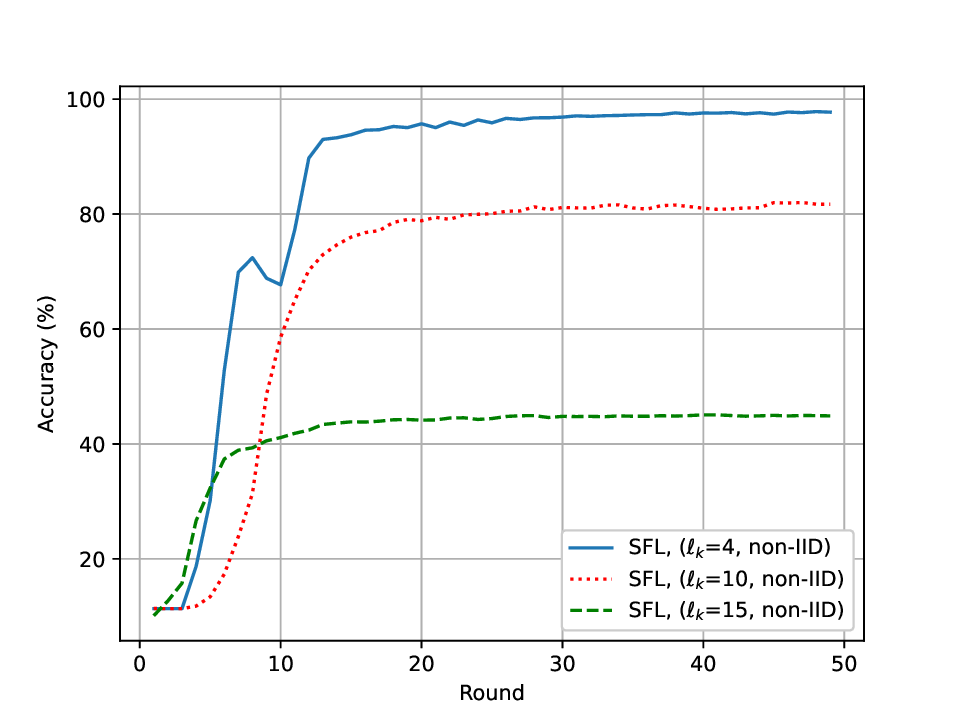}
		}
		\subfigure[Testing accuracy for CIFAR10, IID.]{
			\includegraphics[width=0.46 \linewidth, trim=0.2cm 0.1cm 1cm 0.6cm, clip]{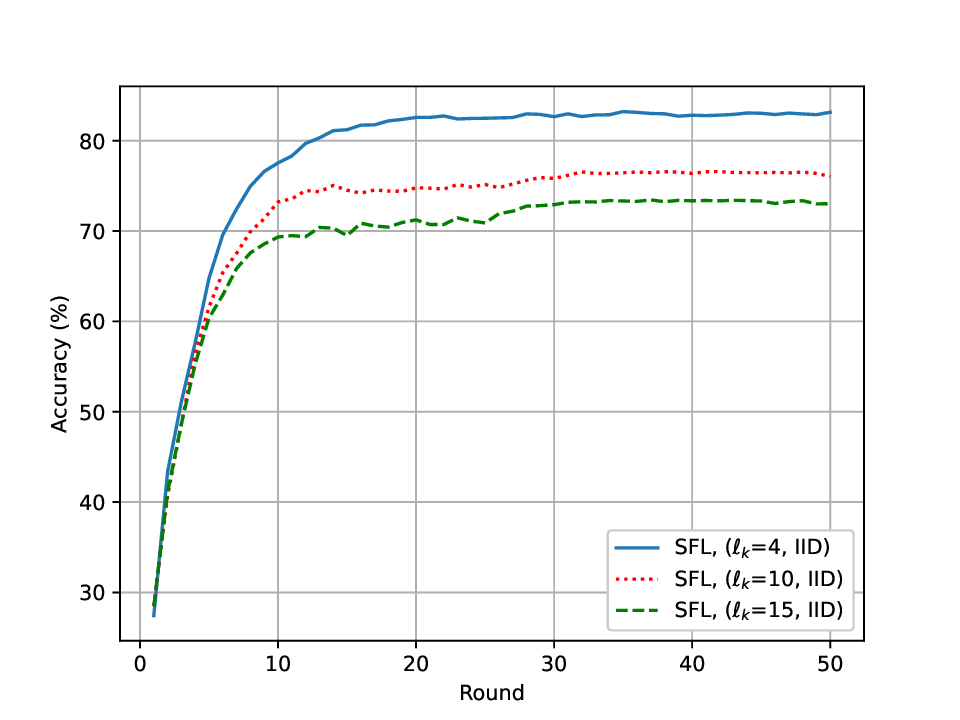}
		}
		\subfigure[Testing accuracy for CIFAR10, non-IID.]{
			\includegraphics[width=0.46 \linewidth, trim=0.2cm 0.1cm 1cm 0.6cm, clip]{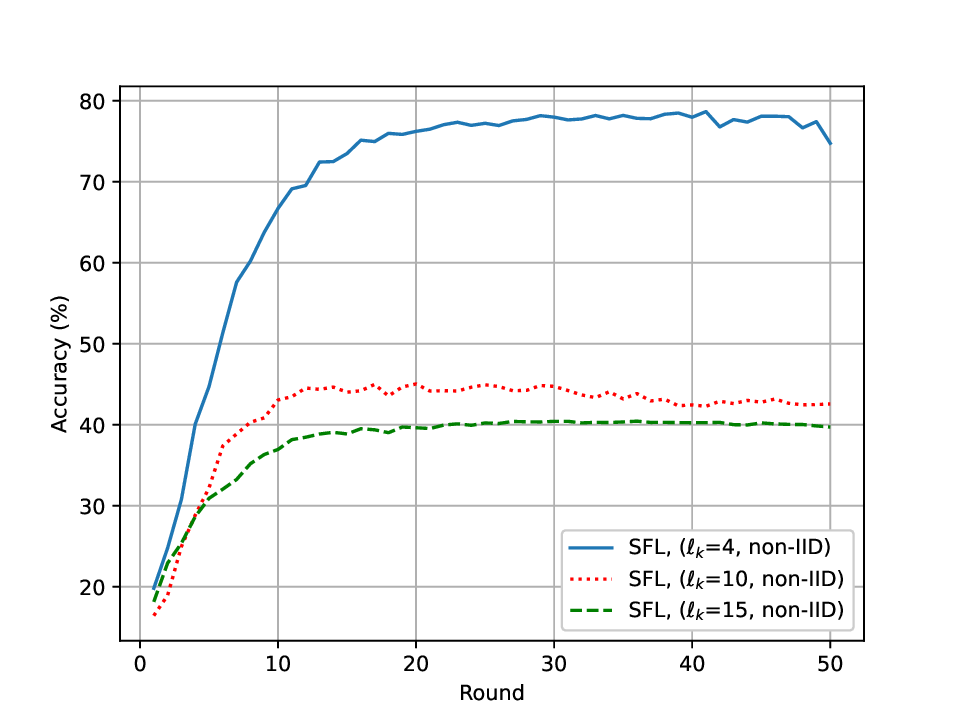}
		}
		\caption{Impact of split point on the testing accuracy. }	
		\label{fig:splitacc}
	\end{figure}

	\begin{figure}[]
	\centering
	\subfigure[Time consumption for AlexNet.]{
		\includegraphics[width=0.46 \linewidth, trim=0.2cm 0.1cm 1cm 0.6cm, clip]{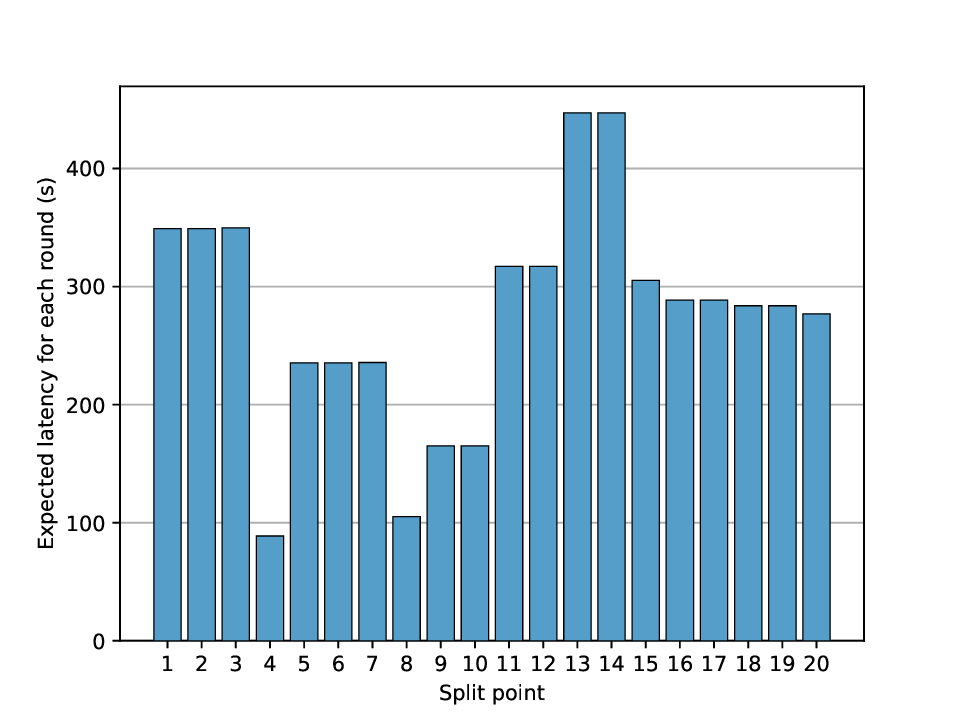}
		\vspace{-0.8cm}
	}
	\ 
	\subfigure[Time consumption for VGG16.]{
		\includegraphics[width=0.46 \linewidth, trim=0.2cm 0.1cm 1cm 0.6cm, clip]{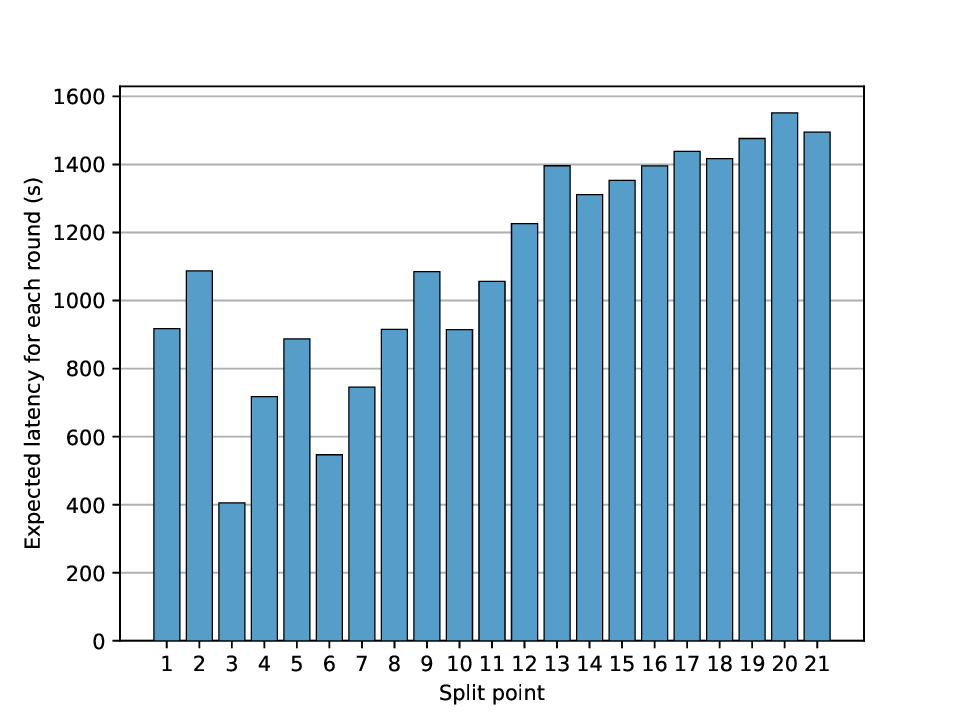}
		\vspace{-0.8cm}
	}	
	\caption{Effect of split points on latency.}
	\label{layer_time}
	\vspace{-0.8cm}
	\end{figure}	
	
	\subsection{Impact of Split Point}

	To investigate the effect of split points on learning performance, we evaluate testing accuracy at different split points. We test both MINST and CIFAR10 dataset, and set $\ell_k$ to be 4, 10 and 15 respectively. Fig. \ref{fig:splitacc} (a) and (b) show the accuracy on MNIST with IID and non-IID, and Fig. \ref{fig:splitacc} (c) and (d) show the performance on CIFAR10. It is obvious that testing accuracy increases as $\ell_k$ decreases, which is due to the reason that more parts of DNN of devices participate in aggregation, and thus the generalization is improved since more data information is used by devices cooperation. Therefore, the accuracy decrease of non-IID is more significant than IID. Fig.\ref{layer_time} shows the latency on different split points of AlexNet and VGG16 for each round. We can observe that the latency is greatly influenced by split point.


	\section{Conclusion}
	\label{sec:Conclusion}
	This paper studied the problem of joint split point selection and bandwidth allocation to minimize the total latency in SFL.
	We proposed an alternating optimization to solve the problem efficiently. 
	We obtained the following insights: First, split points of DNN is crucial in SFL, which affects the time efficiency and model accuracy. Second, compared to FL, the proposed SFL can significantly save the latency. Third, the SFL outperforms FL on test accuracy in the same time.

	\section*{Appendix}
	\appendices


	\subsection{Proof of Theorem \ref{theorem2}}\label{appendix2}
	\subsubsection{additional notation}
	For the convenience of proof, we annotate some variables in the training process. Let $\mathcal{A}$ denote the sub-model aggregation steps, i.e., $\mathcal{A}=\{nE|n=1,2,3,...\}$. When $t\in \mathcal{A}$, $\textbf{h}_{k,t}^B$ should be aggregated. Besides, we define two additional variables $\textbf{v}_{k,t+1}^u$ and $\textbf{v}_{k,t+1}^B$ to be the intermediate result of one iteration update of $\textbf{h}_{k,t}^u$ and $\textbf{h}_{k,t}^B$, and we assume the aggregation ratio for all devices are the same to be $\frac{1}{K}$. Therefore, we have following result:
	\begin{align}
		\begin{aligned}
			\left\{ \begin{array}{l@{}l}
				\textbf{v}_{k,t+1}^u & =\textbf{h}_{k,t}^u-\eta_t\left[\nabla F_k(\textbf{h}_{k,t}^u,\textbf{h}_{k,t}^B,\xi_{k,t})\right]_u;\\
				\textbf{v}_{k,t+1}^B & =\textbf{h}_{k,t}^B-\eta_t\left[\nabla F_k(\textbf{h}_{k,t}^u,\textbf{h}_{k,t}^B,\xi_{k,t})\right]_B,\\
			\end{array}
			\right.
		\end{aligned}
	\end{align}
	and
	\begin{align*}
		\textbf{h}_{k,t+1}^u = \textbf{v}_{k,t+1}^u,
	\end{align*}
	\begin{align}
		\begin{split}
			\textbf{h}_{k,t+1}^B = \left\{ \begin{array}{l@{\quad}l}
				\textbf{v}_{k,t+1}^B & \text{if}\ t+1\notin \mathcal{A};\\
				\sum_{k=1}^K\frac{1}{K} \textbf{v}_{k,t+1}^B & \text{if}\ t+1 \in \mathcal{A},
			\end{array}
			\right.
		\end{split}
	\end{align}
	To further simplify the analysis, we define the total parameter of device $k$ as $\textbf{w}_k=[\textbf{h}_k^u,\ \textbf{h}_k^B]$, and we have  $\overline{\textbf{v}}_{k,t}=\big[\textbf{v}_{k,t}^u,\sum_{k=1}^K\frac{1}{K}\textbf{v}_{k,t}^B\big]$, $\overline{\textbf{w}}_{k,t}=\big[\textbf{h}^u_{k,t},\sum_{k=1}^K\frac{1}{K}\textbf{h}_{k,t}^B\big]$. The gradient of model parameter is denoted as:
	\begin{align}
		\begin{aligned}
			&\textbf{g}_{k,t}=\Big[\big[\nabla F_k(\textbf{w}_{k,t},\xi_{k,t})\big]_u,\big[\sum_{k=1}^K\frac{1}{K}\nabla F_k(\textbf{w}_{k,t},\xi_{k,t})\big]_B\Big],\\
			&\overline{\textbf{g}}_{k,t}=\Big[\big[\nabla F_k(\textbf{w}_{k,t})\big]_u,\big[\sum_{k=1}^K\frac{1}{K}\nabla F_k(\textbf{w}_{k,t})\big]_B\Big].
		\end{aligned}
	\end{align}
	Therefore, $\overline{\textbf{v}}_{k,t+1}=\overline{\textbf{w}}_{k,t}-\eta_t\textbf{g}_{k,t}$, and $\mathbb{E}\textbf{g}_{k,t}=\overline{\textbf{g}}_{k,t}$.
	
	\subsubsection{completing the proof of Theorem \ref{theorem2}}
	Because there are $K$ different lower part of DNN parameters, we analyse the expected average difference between device's parameter and optimal parameter $\mathbb{E}(\sum_{k=1}^K\frac{1}{K}||\overline{\textbf{v}}_{k,t+1}-\textbf{w}^*||^2)$, as following:
	\begin{align}
		\begin{aligned}
			&\mathbb{E}\sum_{k=1}^K\frac{1}{K}||\overline{\textbf{v}}_{k,t+1}-\textbf{w}^*||^2\\
			&=\mathbb{E}\sum_{k=1}^K\frac{1}{K}||\overline{\textbf{w}}_{k,t}-\eta_t\textbf{g}_{k,t}-\textbf{w}^*-\eta_t\overline{\textbf{g}}_{k,t}+\eta_t\overline{\textbf{g}}_{k,t}||^2\\
			&=\mathbb{E}\sum_{k=1}^K\frac{1}{K}\Bigl(\underbrace{||\overline{\textbf{w}}_{k,t}-\textbf{w}^*-\eta_t\overline{\textbf{g}}_{k,t}||^2}_{A_1}\\
			&\quad +\underbrace{2\eta_t\big\langle \overline{\textbf{w}}_{k,t}-\textbf{w}^*-\eta_t\overline{\textbf{g}}_{k,t},\overline{\textbf{g}}_{k,t}-\textbf{g}_{k,t} \big\rangle}_{A_2}\\&\quad+\underbrace{\eta_t^2||\textbf{g}_{k,t}-\overline{\textbf{g}}_{k,t}||^2}_{A_3}\Bigr).
		\end{aligned}
	\end{align} 
	Notice that we have $\mathbb{E}\textbf{g}_{k,t}=\overline{\textbf{g}}_{k,t}$ in the previous analysis, therefore $\mathbb{E}A_2=0$, we first bound $A_1$. 
	\begin{align}
		\begin{aligned}
			A_1&=\sum_{k=1}^K\frac{1}{K}||\overline{\textbf{w}}_{k,t}-\textbf{w}^*-\eta_t\overline{\textbf{g}}_{k,t}||^2\\&=\sum_{k=1}^K\frac{1}{K}\ ||\overline{\textbf{w}}_{k,t}-\textbf{w}^*||^2\\
			&\quad\underbrace{-2\eta_t\big\langle \overline{\textbf{w}}_{k,t}-\textbf{w}^*,\overline{\textbf{g}}_{k,t}\big\rangle}_{B_1}\\
			&\quad+\underbrace{\eta_t^2||\overline{\textbf{g}}_{k,t}||^2}_{B_2}.
		\end{aligned}
	\end{align}
	Now we analyze $B_1$ and $B_2$. For simplicity, we use $\overline{\textbf{h}}_{t}^B$ to denote $\sum_{k=1}^K\frac{1}{K}\textbf{h}_{k,t}^B$
	\begin{align}
		\begin{aligned}
			B_1&=-\sum_{k=1}^K\frac{1}{K}2\eta_t\big\langle \overline{\textbf{w}}_{k,t}-\textbf{w}^*,\overline{\textbf{g}}_{k,t}\big\rangle\\
			&=-\sum_{k=1}^K\frac{1}{K}2\eta_t\Big\langle \overline{\textbf{w}}_{k,t}-\textbf{w}^{*},\\
			&\quad\Big[\big[\nabla F_k(\textbf{w}_{k,t})\big]_u,\big[\sum_{k=1}^K\frac{1}{K}\nabla F_k(\textbf{w}_{k,t})\big]_B\Big]\Big\rangle\\
			&=-\sum_{k=1}^K\frac{1}{K}2\eta_t\Big\langle \overline{\textbf{w}}_{k,t}-\textbf{w}^{*},\Big[\big[\nabla F_k(\textbf{w}_{k,t})\big]_u,\big[\nabla F_k(\textbf{w}_{k,t})\big]_B\Big]\Big\rangle\\
			&=-\sum_{k=1}^K\frac{1}{K}2\eta_t\Big\langle \overline{\textbf{w}}_{k,t}-\textbf{w}_{k,t},\nabla F_k(\textbf{w}_{k,t})\Big\rangle\\
			&\quad-\sum_{k=1}^K\frac{1}{K}2\eta_t\Big\langle \textbf{w}_{k,t}-\textbf{w}^*,\nabla F_k(\textbf{w}_{k,t})\Big\rangle\\
			&\leq\eta_t\sum_{k=1}^K\frac{1}{K}\Big(\frac{1}{\eta_t}\big\Vert\overline{\textbf{w}}_{k,t}-\textbf{w}_{k,t}\big\Vert^2+\eta_t\big\Vert\nabla F_k(\textbf{w}_{k,t})\big\Vert^2\Big)\\
			&\quad-2\eta_t\sum_{k=1}^K\frac{1}{K}\big\langle\textbf{w}_{k,t}-\textbf{w}^{*},\nabla F_k(\textbf{w}_{k,t})\big\rangle.
		\end{aligned}
	\end{align}
	The last inequality is due to the AM-GM inequality. By the µ-strong convexity of $F_k(\cdot)$, we have
	\begin{align}
		\begin{aligned}
			&-\big\langle\textbf{w}_{k,t}-\textbf{w}^{*},\nabla F_k(\textbf{w}_{k,t})\big\rangle\\
			&\quad\quad\leq-\Big(F_k(\textbf{w}_{k,t})-F^*+\frac{\mu}{2}\big\Vert\textbf{w}_{k,t}-\textbf{w}^*\big\Vert^2\Big).
		\end{aligned}
	\end{align}
	Therefore, $B_1$ satisfies:
	\begin{align}
		\begin{aligned}
			B_1=&-\sum_{k=1}^K\frac{1}{K}2\eta_t\big\langle \overline{\textbf{w}}_{k,t}-\textbf{w}^*,\overline{\textbf{g}}_{k,t}\big\rangle\\
			\leq&\eta_t\sum_{k=1}^K\frac{1}{K}\Big(\frac{1}{\eta_t}\big\Vert\overline{\textbf{w}}_{k,t}-\textbf{w}_{k,t}\big\Vert^2+\eta_t\big\Vert\nabla F_k(\textbf{w}_{k,t})\big\Vert^2\Big)\\
			&-2\eta_t\sum_{k=1}^K\frac{1}{K}\Big(F_k(\textbf{w}_{k,t})-F^*+\frac{\mu}{2}\big\Vert\textbf{w}_{k,t}-\textbf{w}^*\big\Vert^2\Big).
		\end{aligned}
	\end{align}
	For $B_2$ we have 
	\begin{align}\label{eqn:B_2}
		\begin{aligned}
			B_2=&\sum_{k=1}^K\frac{1}{K}\eta_t^2\big{\Vert}\overline{\textbf{g}}_{k,t}\big{\Vert}^2\\
			=&\eta_t^2\sum_{k=1}^K\frac{1}{K}\big{\Vert}([\nabla F_k(\textbf{w}_{k,t})]_u,\sum_{k=1}^K\frac{1}{K}[\nabla F_k(\textbf{w}_{k,t})]_B)\big{\Vert}^2\\
			=&\eta_t^2\sum_{k=1}^K\frac{1}{K}\big{\Vert}[\nabla F_k(\textbf{w}_{k,t})]_u\big{\Vert}^2+\eta_t^2\big{\Vert}\sum_{k=1}^K\frac{1}{K}[\nabla F_k(\textbf{w}_{k,t})]_B\big{\Vert}^2\\
			\leq&\eta_t^2\sum_{k=1}^K\frac{1}{K}\big{\Vert}[\nabla F_k(\textbf{w}_{k,t})]_u\big{\Vert}^2+\eta_t^2\sum_{k=1}^K\frac{1}{K^2}\big{\Vert}[\nabla F_k(\textbf{w}_{k,t})]_B\big{\Vert}^2\\
			\leq&\eta_t^2\Big(lZ^2+\frac{1}{K}(L-l)Z^2\Big).
		\end{aligned}
	\end{align}
	Where the last inequality is due to the assumption that the squared norm of the gradient for each layer has upper bound $Z^2$. Therefore, by the analysis of $B_1$ and $B_2$, it follows that
	\begin{align}
		\begin{aligned}
			A_1=&\sum_{k=1}^K\frac{1}{K}\big{\Vert}\overline{\textbf{w}}_{k,t}-\textbf{w}^*-\eta_t\overline{\textbf{g}}_{k,t}\big{\Vert}^2\\
			\leq&\sum_{k=1}^K\frac{1}{K}\big{\Vert}\overline{\textbf{w}}_{k,t}-\textbf{w}^*\big{\Vert}^2\\
			&+\eta_t\sum_{k=1}^K\frac{1}{K}\Big(\frac{1}{\eta_t}\big\Vert\overline{\textbf{w}}_{k,t}-\textbf{w}_{k,t}\big\Vert^2+\eta_t\big\Vert\nabla F_k(\textbf{w}_{k,t})\big\Vert^2\Big)\\
			&-2\eta_t\sum_{k=1}^K\frac{1}{K}\Big(F_k(\textbf{w}_{k,t})-F^*+\frac{\mu}{2}\big\Vert\textbf{w}_{k,t}-\textbf{w}^*\big\Vert^2\Big)\\
			&+\eta_t^2\Big(lZ^2+\frac{1}{K}(L-l)Z^2\Big)\\
			\leq&(1-\mu\eta_t)\sum_{k=1}^K\frac{1}{K}\big{\Vert}\overline{\textbf{w}}_{k,t}-\textbf{w}^*\big{\Vert}^2\\
			&+\sum_{k=1}^K\frac{1}{K}\big\Vert\overline{\textbf{w}}_{k,t}-\textbf{w}_{k,t}\big\Vert^2-2\eta_t\sum_{k=1}^K\frac{1}{K}\big(F_k(\textbf{w}_{k,t})-F^*\big)\\
			&+\eta_t^2\Big(lZ^2+\frac{1}{K}(L-l)Z^2\Big)+\eta_t^2\sum_{k=1}^K\frac{1}{K}\big\Vert\nabla F_k(\textbf{w}_{k,t})\big\Vert^2.\\
		\end{aligned}
	\end{align}
	By the $\beta$-smoothness of $F_k(\cdot)$, we have 
	\begin{align}\label{eqn:beta-smooth}
			\big{\Vert}\nabla F_k(\textbf{w}_{k,t})\big{\Vert}^2\leq 2\beta\big(F_k(\textbf{w}_{k,t})-F_k^*\big).
	\end{align}
	Therefore, for $A_1$ we have
%
%

	\begin{align}
		\begin{aligned}
			&A_1\\
			&\leq(1-\mu\eta_t)\sum_{k=1}^K\frac{1}{K}\big{\Vert}\overline{\textbf{w}}_{k,t}-\textbf{w}^*\big{\Vert}^2+\sum_{k=1}^K\frac{1}{K}\big\Vert\overline{\textbf{w}}_{k,t}-\textbf{w}_{k,t}\big\Vert^2\\
			&\quad+\eta_t^2\Big(lZ^2+\frac{1}{K}(L-l)Z^2\Big)\\
			&\quad\underbrace{+2\beta\eta_t^2\sum_{k=1}^K\frac{1}{K}\big(F_k(\textbf{w}_{k,t})-F_k^*\big)-2\eta_t\sum_{k=1}^K\frac{1}{K}\big(F_k(\textbf{w}_{k,t})-F^*\big)}_C.
		\end{aligned}
	\end{align}
	Next we bound $C$. Define $\gamma_t=2\eta_t(1-\beta\eta_t)$, we have 
	\begin{align}
		\begin{aligned}
			C=&-2\eta_t(1-\beta\eta_t)\sum_{k=1}^K\frac{1}{K}\big(F_k(\textbf{w}_{k,t})-F_k^*\big)\\
			&+2\eta_t\sum_{k=1}^K\frac{1}{K}\big(F^*-F_k^*\big)\\
			=&-\gamma_t\sum_{k=1}^K\frac{1}{K}\big(F_k(\textbf{w}_{k,t})-F^*\big)\\
			&+(2\eta_t-\gamma_t)\sum_{k=1}^K\frac{1}{K}\big(F^*-F_k^*\big)\\
			=&\underbrace{-\gamma_t\sum_{k=1}^K\frac{1}{K}\big(F_k(\textbf{w}_{k,t})-F^*\big)}_D+2\beta\eta_t^2\Gamma,
		\end{aligned}
	\end{align}
	where $\Gamma=\sum_{k=1}^K\frac{1}{K}\big(F^*-F_k^*\big)$. To bound $D$, we have
	\begin{align}
		\begin{aligned}
			&\sum_{k=1}^K\frac{1}{K}\big(F_k(\textbf{w}_{k,t})-F_k^*\big)\\
			&=\sum_{k=1}^K\frac{1}{K}\big(F_k(\textbf{w}_{k,t})-F_k(\overline{\textbf{w}}_{k,t})\big)+\sum_{k=1}^K\frac{1}{K}\big(F_k(\overline{\textbf{w}}_{k,t})-F^*\big)\\
			&\geq\sum_{k=1}^K\frac{1}{K}\big{\langle}\nabla F_k(\overline{\textbf{w}}_{k,t}),\textbf{w}_{k,t}-\overline{\textbf{w}}_{k,t}\big{\rangle}+\sum_{k=1}^K\frac{1}{K}\big(F_k(\overline{\textbf{w}}_{k,t})-F^*\big)\\
			&\geq-\frac{1}{2}\sum_{k=1}^K\frac{1}{K}\big(\eta_t\big{\Vert}\nabla F_k(\overline{\textbf{w}}_{k,t})\big{\Vert}^2+\frac{1}{\eta_t}\big{\Vert}\textbf{w}_{k,t}-\overline{\textbf{w}}_{k,t}\big{\Vert}^2\big)\\&\quad+\sum_{k=1}^K\frac{1}{K}\big(F_k(\overline{\textbf{w}}_{k,t})-F^*\big)\\
			&\geq-\sum_{k=1}^K\frac{1}{K}\big(\eta_t\beta(F_k(\overline{\textbf{w}}_{k,t})-F^*)+\frac{1}{2\eta_t}\big{\Vert}\textbf{w}_{k,t}-\overline{\textbf{w}}_{k,t}\big{\Vert}^2\big)\\&\quad+\sum_{k=1}^K\frac{1}{K}\big(F_k(\overline{\textbf{w}}_{k,t})-F^*\big),
		\end{aligned}
	\end{align}
	where the first inequality is due to the convexity of $F(\cdot)$, the second inequality is from AM-GM inequality, and the last inequality results from \eqref{eqn:beta-smooth}. Therefore, for $C$ we have
	\begin{align}
		\begin{aligned}
			C\leq&\gamma_t\sum_{k=1}^K\frac{1}{K}\big(\eta_t\beta(F_k(\overline{\textbf{w}}_{k,t})-F^*)+\frac{1}{2\eta_t}\big{\Vert}\textbf{w}_{k,t}-\overline{\textbf{w}}_{k,t}\big{\Vert}^2\big)\\&-\gamma_t\sum_{k=1}^K\frac{1}{K}\big(F_k(\overline{\textbf{w}}_{k,t})-F^*\big)+2\beta\eta_t^2\Gamma\\
			=&\gamma_t(\eta_t\beta-1)\sum_{k=1}^K\frac{1}{K}(F_k(\overline{\textbf{w}}_{k,t})-F^*)\\
			&+\frac{\gamma_t}{2\eta_t}\sum_{k=1}^K\frac{1}{K}\big{\Vert}\textbf{w}_{k,t}-\overline{\textbf{w}}_{k,t}\big{\Vert}^2	+(2\beta\eta_t^2+\gamma_t\eta_t\beta)\Gamma\\
			\leq&\sum_{k=1}^K\frac{1}{K}\big{\Vert}\textbf{w}_{k,t}-\overline{\textbf{w}}_{k,t}\big{\Vert}^2+4\beta\eta_t^2\Gamma,
		\end{aligned}
	\end{align}
	where in the last inequality, we use the fact that $\eta_t\beta-1\leq 0$, $F_k(\overline{\textbf{w}}_{k,t})-F^*\geq 0$, $\Gamma\geq 0$, and $\gamma_t\leq 2\eta_t$.
	
	Therefore,
	\begin{align}
		\begin{aligned}
			A_1=&\sum_{k=1}^K\frac{1}{K}\big{\Vert}\overline{\textbf{w}}_{k,t}-\textbf{w}^*-\eta_t\overline{\textbf{g}}_{k,t}\big{\Vert}^2\\
			\leq&(1-\mu\eta_t)\sum_{k=1}^K\frac{1}{K}\big{\Vert}\overline{\textbf{w}}_{k,t}-\textbf{w}^*\big{\Vert}^2\\
			&+2\sum_{k=1}^K\frac{1}{K}\big\Vert\overline{\textbf{w}}_{k,t}-\textbf{w}_{k,t}\big\Vert^2+4\beta\eta_t^2\Gamma.
		\end{aligned}
	\end{align}
	Next we bound the expectation of $A_3$ and $A_1$. For $A_3$, 
	\begin{align}
		\begin{aligned}
			\mathbb{E}A_3=&\mathbb{E}\sum_{k=1}^K\frac{1}{K}\big{\Vert}\textbf{g}_{k,t}-\overline{\textbf{g}}_{k,t}\big{\Vert}^2\\
			=&\sum_{k=1}^K\frac{1}{K}\mathbb{E}\big{\Vert}\big[\nabla F_k(\textbf{w}_{k,t},\xi_{k,t})\big]_u-\big[\nabla F_k(\textbf{w}_{k,t})\big]_u\big{\Vert}^2\\
			&+\sum_{k=1}^K\frac{1}{K}\mathbb{E}\big{\Vert}\sum_{k=1}^K\frac{1}{K}\big(\big[\nabla F_k(\textbf{w}_{k,t},\xi_{k,t})\big]_B\\
			&\quad\quad-\big[\nabla F_k(\textbf{w}_{k,t})\big]_B\big)\big{\Vert}^2\\
			=&\sum_{k=1}^K\frac{1}{K}\mathbb{E}\big{\Vert}\big[\nabla F_k(\textbf{w}_{k,t},\xi_{k,t})\big]_u-\big[\nabla F_k(\textbf{w}_{k,t})\big]_u\big{\Vert}^2\\
			&+\sum_{k=1}^K\frac{1}{K^2}\mathbb{E}\big{\Vert}\big[\nabla F_k(\textbf{w}_{k,t},\xi_{k,t})\big]_B-\big[\nabla F_k(\textbf{w}_{k,t})\big]_B\big{\Vert}^2\\
			\leq&l\sigma^2+\frac{1}{K}(L-l)\sigma^2.
		\end{aligned}
	\end{align}
	
	To bound $\mathbb{E}\sum_{k=1}^K\frac{1}{K}\big\Vert\overline{\textbf{w}}_t^B-\textbf{w}_{k,t}^B\big\Vert^2$, we assume $E$ steps between two aggregations, and let $t_0$ be the first iteration from some aggregation. Therefore, for any $t\geq 0$, there exists a $t_0 \geq t$, such that $t-t_0 \leq E-1 $ and $\textbf{w}_{k,t_0}=\overline{\textbf{w}}_{k,t_0}$ then
	
	\begin{align}
		\begin{aligned}
			&\mathbb{E}\sum_{k=1}^K\frac{1}{K}\big\Vert\overline{\textbf{w}}_{k,t}-\textbf{w}_{k,t}\big\Vert^2\\
			&=\mathbb{E}\sum_{k=1}^K\frac{1}{K}\big\Vert(\textbf{w}_{k,t}-\overline{\textbf{w}}_{k,t_0})-(\overline{\textbf{w}}_{k,t}-\overline{\textbf{w}}_{k,t_0})\big\Vert^2\\
			&\leq\mathbb{E}\sum_{k=1}^K\frac{1}{K}\big\Vert\textbf{w}_{k,t}-\overline{\textbf{w}}_{k,t_0}\big\Vert^2\\
			&\leq\sum_{k=1}^K\frac{1}{K}\mathbb{E}\sum_{t=t_0}^{t}(E-1)\eta_t^2\big{\Vert}\nabla F_k(\textbf{w}_{k,t},\xi_{k,t})\big{\Vert}^2\\
			&\leq(E-1)^2\eta_t^2LZ^2,
		\end{aligned}
	\end{align}
	where the first inequality is due to $\mathbb{E}\big{\Vert}X-\mathbb{E}X\big{\Vert}^2\leq\mathbb{E}\big{\Vert}X\big{\Vert}^2$, and the second inequality is due to Jensen inequality:
	\begin{align*}
		\big\Vert\textbf{w}_{k,t}-\overline{\textbf{w}}_{k,t_0}\big\Vert^2&=\big{\Vert}\sum_{t=t_0}^{t}\eta_t\nabla F_k(\textbf{w}_{k,t},\xi_{k,t})\big{\Vert}^2\\
		&\leq (t-t_0)\sum_{t=t_0}^{t}\eta_t^2\big{\Vert}\nabla F_k(\textbf{w}_{k,t},\xi_{k,t})\big{\Vert}^2.
	\end{align*}
	The last inequality is due to the assumption. In this way, we have
	\begin{align}
		\begin{aligned}
			&\mathbb{E}\sum_{k=1}^K\frac{1}{K}||\overline{\textbf{v}}_{k,t+1}-\textbf{w}^*||^2\\
			&\leq(1-\mu\eta_t)\sum_{k=1}^K\frac{1}{K}\mathbb{E}\big{\Vert}\overline{\textbf{w}}_{k,t}-\textbf{w}^*\big{\Vert}^2\\
			&\quad+2(E-1)^2\eta_t^2LZ^2+4\beta\eta_t^2\Gamma\\
			&\quad+\eta_t^2\big(lZ^2+\frac{1}{K}(L-l)Z^2\big)\\
			&\quad+\eta_t^2\big(l\sigma^2+\frac{1}{K}(L-l)\sigma^2\big).
		\end{aligned}
	\end{align}
	Let $P=2(E-1)^2LZ^2+4\beta\Gamma+lZ^2+\frac{1}{K}(L-l)Z^2+l\sigma^2+\frac{1}{K}(L-l)\sigma^2$, and $\Delta_{t}=\mathbb{E}\sum_{k=1}^K\frac{1}{K}\big{\Vert}\overline{\textbf{w}}_{k,t}-\textbf{w}^*\big{\Vert}^2$. We assume $\eta_t=\frac{\lambda}{t+\gamma}$ for some $\lambda>\frac{1}{\mu}$ and $\gamma>0$. Now we prove $\Delta_{t}\leq\frac{v}{\gamma+t}$ by induction, where $v=\max\big\{\frac{\lambda^2P}{\lambda\mu-1},(\gamma+1)\Delta_{1}\big\}$. Obviously, the definition of $v$ ensures it holds for $t=1$. For some $t$, 
	
	\begin{align}
		\begin{aligned}
			\Delta_{t+1}\leq& (1-\mu\eta_t)\Delta_{t}+\eta_t^2P\\
			\leq&\Bigg(1-\frac{\lambda\mu}{t+\gamma}\Bigg)\frac{v}{t+\gamma}+\frac{\lambda^2P}{(t+\gamma)^2}\\
			=&\frac{t+\gamma-1}{(t+\gamma)^2}v+\Bigg[\frac{\lambda^2P}{(t+\gamma)^2}-\frac{\lambda\mu-1}{(t+\gamma)^2}v\Bigg]\\
			\leq&\frac{v}{\gamma+t+1}.
		\end{aligned}
	\end{align}
	Then by the $\beta$-smoothness of $F(\cdot)$, 
	\begin{align*}
		\mathbb{E}\sum_{k=1}^{K}F_k(\overline{\textbf{w}}_{k,t})-F^*\leq\frac{\beta}{2}\Delta_{t}\leq\frac{\beta}{2}\frac{v}{\gamma+t}.
	\end{align*}
	Specifically, let $\lambda=\frac{2}{\mu}$, $\alpha=\frac{\beta}{\mu}$ and $\gamma = \max\{8\alpha, E\}-1$, then 
	\begin{align}
		\begin{aligned}
			v&=\max\big\{\frac{\lambda^2P}{\lambda\mu-1},(\gamma+1)\Delta_{1}\big\}\\
			&\leq\frac{\lambda^2P}{\lambda\mu-1}+(\gamma+1)\Delta_{1}\\
			&\leq\frac{4P}{\mu^2}+(\gamma+1)\Delta_{1}.
		\end{aligned}
	\end{align}
Therefore, we have
	\begin{align}
		\begin{aligned}
			\mathbb{E}\frac{1}{K}\sum_{k=1}^{K}F_k(\overline{\textbf{w}}_{k,t})-F^*&\leq\frac{\beta}{2}\frac{v}{\gamma+t}\\
			&\leq\frac{\alpha}{\gamma+t}\Bigg(\frac{2P}{\mu}+\frac{\mu(\gamma+1)}{2}\Delta_{1}\Bigg).
		\end{aligned}
	\end{align}

	\subsection{Proof of Theorem \ref{theorem1}}
	Obviously, from equation \eqref{eqn:shannon} we have that transmission rate monotonically increases with $b_k$. Therefore, the transmission latency reduces if device is allocated with more bandwidth. In this way, the devices which have higher latency should be allocated with more bandwidth, from which have lower latency. To the end, the optimal solution of P2 can be achieved if and only if all bandwidth is allocated and all devices have the same finishing time. As a result, the optimal bandwidth allocation ration should satisfy the following equation:
	\begin{align}
		\label{eq:P_of_t1}
		\begin{aligned}
			\left\{ \begin{array}{l}
				\tau_{\ell_k^*}^{cp}+\frac{D_{\ell_k^*}}{b_kW\log_2(1+\frac{p_k|g_k|^2}{N_0})}=\tau_k^*,\ k\in\mathcal{K};\\
				\sum_{k=1}^{K}b_k=1;\\
				0<b_k<1.\\			
			\end{array}
			\right.
		\end{aligned}
	\end{align}
	
	Solving \eqref{eq:P_of_t1}, we easily have :
	\begin{align}
		b_k^* = \frac{D_{\ell_k^*}}{[\tau_k^*-\tau_{\ell_k^*}^{cp}]W\log_2(1+\frac{p_k|g_k|^2}{N_0})}\nonumber.
	\end{align}
	In this way, we only need to choose the appropriate value of $\tau_k^*$.

\end{document}